\documentclass{article}
\pdfoutput=1

\usepackage{caption}
\usepackage{subcaption}

\usepackage[preprint]{neurips_2023}

\usepackage[utf8]{inputenc} 
\usepackage[T1]{fontenc}    
\usepackage{hyperref}       
\usepackage{url}            
\usepackage{booktabs}       
\usepackage{amsfonts, amsmath}       
\usepackage{nicefrac}       
\usepackage{microtype}      
\usepackage{xcolor}         
\usepackage{bbm}
\usepackage{times}
\usepackage{algorithm, algorithmic}
\usepackage{mathtools}
\usepackage{latexsym}
\usepackage{relsize}
\usepackage{amsthm}
\usepackage{tabularx}

\newcommand{\sym}{\mathop{\mbox{\rm sym}}}
\usepackage{newfloat}
\DeclareFloatingEnvironment[
	fileext=lob,
	listname={List of Boxes},
	name={MP},
	placement=htp,
]{optbox}
\newcommand{\mM}{{\mathcal M}}
\newcommand{\mMrk}{{\mathcal M}^{\mathrm{rk}}_k}

\newcommand{\mMmax}{{\mathcal M}^{\mathrm{max}}_K}
\newcommand{\mMtr}{{\mathcal M}^{\mathrm{tr}}_\tau}
\newcommand{\NP}{\mathsf{NP}}

\newcommand{\sampc}{\mathsf{sc}}
\newcommand{\FC}{\mathsf{FullComp}}

\newcommand{\mnorm}[1]{\| #1 \|_{\max}}
\newcommand{\frnorm}[1]{\| #1 \|_{\mathrm{fr}}}
\newcommand{\trnorm}[1]{\| #1 \|_{\mathrm{tr}}}

\def\inds{{\mathcal{X}}}

\def\V{{\mathcal V}}
\def\M{{\mathcal M}}

\newcommand{\F}{\mathbb{F}}

\newcommand{\defeq}{\stackrel{\text{def}}{=}}

\def\X{{\mathcal X}}
\def\H{{\mathcal H}}
\def\Y{{\mathcal Y}}
\def\W{{\mathcal W}}
\def\R{{\mathcal R}}

\def\sym{{\text Sym}}

\newcommand{\D}{\mathcal{D}}

\newcommand{\A}{\mathcal{A}}

\newcommand{\rank}{\mathrm{rank}}

\newcommand{\sign}{\text{ } \mathrm{sign}}

\newcommand{\bzero}{\ensuremath{\mathbf 0}}

\newcommand{\K}{\ensuremath{\mathcal K}}

\def\C{{\mathcal C}}

\def\bone{\mathbf{1}}

\def\regret{\mbox{{Regret}}}
\def\spregret{\mbox{{SP-Regret}}}
\def\gregret{\mbox{{Game-Regret}}}

\newcommand{\ignore}[1]{}
\newcommand{\eh}[1]{\noindent{\textcolor{blue}{\{{\bf EH:} \em #1\}}}}

\theoremstyle{plain}
\newtheorem{theorem}{Theorem}
\newtheorem{lemma}[theorem]{Lemma}
\newtheorem{corollary}[theorem]{Corollary}
\newtheorem{proposition}[theorem]{Proposition}

\newtheorem{assumption}{Assumption}

\newtheorem*{theorem*}{Theorem}
\newtheorem*{lemma*}{Lemma}
\newtheorem*{corollary*}{Corollary}
\newtheorem*{proposition*}{Proposition}
\newtheorem*{claim*}{Claim}
\newtheorem*{fact*}{Fact}
\newtheorem*{observation*}{Observation}
\newtheorem*{assumption*}{Assumption}

\theoremstyle{definition}
\newtheorem{definition}[theorem]{Definition}
\newtheorem{remark}[theorem]{Remark}

\newtheorem*{definition*}{Definition}
\newtheorem*{remark*}{Remark}
\newtheorem*{example*}{Example}

\DeclareMathAlphabet{\mathbfsf}{\encodingdefault}{\sfdefault}{bx}{n}

\DeclareMathOperator*{\argmin}{arg\,min}
\DeclareMathOperator*{\argmax}{arg\,max}

\DeclareMathOperator*{\conv}{conv}

\let\Pr\relax
\DeclareMathOperator{\Pr}{\mathbb{P}}

\def\mA{{\mathcal A}}

\newcommand{\mycases}[4]{{
\left\{
\begin{array}{ll}
    {#1} & {\;\text{#2}} \\[1ex]
    {#3} & {\;\text{#4}}
\end{array}
\right. }}

\newcommand{\norm}[1]{\|#1\|}

\newcommand{\E}{\mathbb{E}}

\newcommand{\trace}{\mathrm{tr}}

\newcommand{\tr}{^{\mkern-1.5mu\mathsf{T}}}

\newcommand{\reals}{\mathbb{R}}
\newcommand{\eps}{\varepsilon}

\renewcommand{\leq}{~\le~}
\renewcommand{\geq}{~\ge~}

\let\oldtfrac\tfrac
\renewcommand{\tfrac}[2]{\smash{\oldtfrac{#1}{#2}}}

\let\nablaold\nabla
\renewcommand{\nabla}{\nablaold\mkern-2.5mu}

\newcommand{\mF}{\mathcal{F}}

\makeatletter
\newtheorem*{rep@theorem}{\rep@title}
\newcommand{\newreptheorem}[2]{%
\newenvironment{rep#1}[1]{%
 \def\rep@title{#2 \ref{##1}}%
 \begin{rep@theorem}}%
 {\end{rep@theorem}}}
\makeatother
\newreptheorem{theorem}{Theorem}

\makeatletter
\newtheorem*{rep@proposition}{\rep@title}
\newcommand{\newrepproposition}[2]{%
\newenvironment{rep#1}[1]{%
 \def\rep@title{#2 \ref{##1}}%
 \begin{rep@proposition}}%
 {\end{rep@proposition}}}
\makeatother
\newreptheorem{proposition}{Proposition}

\makeatletter
\newtheorem*{rep@lemma}{\rep@title}
\newcommand{\newreplemma}[2]{%
\newenvironment{rep#1}[1]{%
 \def\rep@title{#2 \ref{##1}}%
 \begin{rep@lemma}}%
 {\end{rep@lemma}}}
\makeatother
\newreptheorem{lemma}{Lemma}

\title{Partial Matrix Completion}

\author{%
  Elad Hazan \thanks{Authors ordered alphabetically}\\
  Princeton University\\
  Google DeepMind \\
  \And
  Adam Tauman Kalai \footnotemark[1]\\
  Microsoft Research \\
  \AND
  Varun Kanade \footnotemark[1]\\
  University of Oxford \\
  \And
  Clara Mohri \footnotemark[1]\\
  Harvard University \\
  \And
  Y. Jennifer Sun \footnotemark[1]\\
  Princeton University \\
  Google DeepMind \\
}

\begin{document}

\maketitle

\begin{abstract}
	The matrix completion problem aims to reconstruct a low-rank
	matrix based on a revealed set of possibly noisy entries. Prior works
	consider completing the entire matrix with generalization error guarantees. However, the completion accuracy can be drastically different over different entries. This work establishes a new framework of \emph{partial matrix completion}, where the goal is to identify a large subset of the entries that can be completed with high confidence. We propose an efficient algorithm with the following provable guarantees.  Given access to samples from an unknown and arbitrary distribution, it guarantees: (a) high accuracy over completed entries, and (b) high coverage of the underlying distribution. We also consider an online learning variant of this problem, where we propose a low-regret algorithm based on iterative gradient updates. Preliminary empirical evaluations are included. 
\end{abstract}

\section{Introduction}

In the classical matrix completion problem, a subset of entries of a matrix are revealed and the goal is to reconstruct the full matrix. This is in general impossible, but if the matrix is assumed to be low rank and the distribution over revealed entries is uniformly random, then it can be shown that reconstruction is possible. A common application of matrix completion is in recommendation systems. For example, the rows of the matrix can correspond to users, the columns to movies, and an entry of the matrix is the preference score for the user over the corresponding movie. The completed entries can then be used to predict user preferences over unseen movies. The low rank assumption in this case is justified if the preference of users over movies is mostly determined by a small number of latent factors such as the genre, director, artistic style, and so forth. 

However, in many cases, it is both infeasible and unnecessary to complete the entire matrix. First consider the trivial example of movie recommendations with users and movies coming from two countries $A$
and $B$, where each user rates random movies only from their country. Without any
cross ratings of users from country $A$ on movies from country $B$ or vice
versa, it is impossible to accurately complete these 		entries. A solution here
is straightforward: partition the users and movies into their respective groups
and then complete only the part of the matrix corresponding to user ratings of
movies from their own country.

In reality, many users have categories of movies with few or no ratings based on genre, language, or time period. For such users, it is difficult to accurately complete unrated movie categories which they do not like to watch from those which they have not even been exposed to. Thus it may be preferable to abstain from making rating predictions for such users on these unpredictable categories. This is further complicated by the fact that such categories are not crisply defined and relevant categories may vary across user demographics like country and age. 

\begin{figure*}[!ht]
    \centering
    \includegraphics[width=\textwidth]{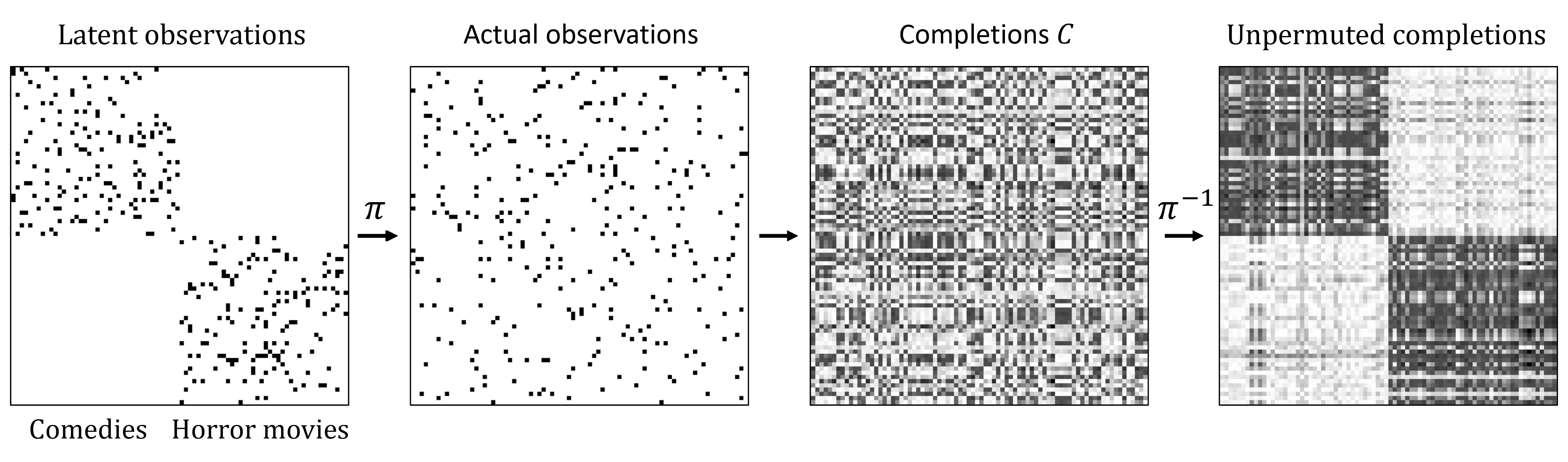}
    \caption{A simple example where some users rate only comedies and some rate only horror movies. However, these two groups are unknown and the rows and columns are permuted according to $\pi = \pi_C \circ \pi_R$. The partial completions made by our algorithm, shown on simulated data, correspond to completing the ratings for the half of the matrix as one would like. Note that our algorithms can handle \textit{arbitrary} revealed subsets, not only stochastic block models. The experimental setup and more preliminary experiments are described in Section \ref{sec:experiments}.}
    \label{fig:my_label}
\end{figure*}

This motivates a more general question: can we identify a subset of the matrix that can be completed by the existing observations with high confidence? We call this problem \textbf{partial matrix completion}. Let $M^\star$ be a target matrix. A \textit{non-convex} formulation would be to find a matrix $\hat{M}$ and a set of entries $C \subseteq [m] \times [n]$ which maximizes:
\begin{equation}
\max_{C \subseteq [m]\times [n]} |C| ~\text{ s.t. }~ \frac{1}{|C|} \sum_{(i, j) \in C} (M_{ij}^{\star}-\hat{M}_{ij})^2 \le \eps.    \label{eq:nonconvex-pmc}
\end{equation}
This is non-convex because the set $C$ is discrete and also because the choice of $\hat{M}$ interacts $C$. Fortunately, there are already several matrix completion algorithms which output $\hat{M}$ that guarantee low ``generalization error'' with respect to future examples \textit{from the same distribution of entries} as those from the training distribution, though they are not accurate over the entire matrix. We will fix any one of these completions $\hat{M}$.

We thus relax the goal of partial matrix completion to identify a matrix $C$ which gives a confidence score for each entry between 0 and 1. A value of 0 indicates that this entry should not be completed, whereas 1 indicates we have absolute confidence. We call this matrix the {\bf confidence matrix}.  Ideally, we would like that the confidence matrix to have a value of 1 on all entries that are well supported by the observation distribution. 

We formalize this problem through the following definitions. The precise definitions of some of the terms will be introduced in subsequent sections. 
Let $M^\star$ be a target matrix and $\hat{M}$ be a completion given by a matrix completion algorithm that guarantees low generalization error over future examples from the same entry distribution (formal definition given in Def.~\ref{defn:fullcompalg}), and let $\eps>0$.  The  \textit{coverage} of $C$ is defined to be $\|C\|_1 \defeq \sum_{ij} |C_{ij}|$.
\begin{definition} [Partial matrix completion problem] 
\label{def:fractional-pmc}
Let $M^\star$ be a target matrix and $\hat{M}$ be a full completion (Def.~\ref{defn:fullcompalg}), and $\eps>0$ be an error tolerance. Find a confidence matrix $C\in[0,1]^{m\times n}$ with maximal coverage, i.e. 
\begin{align*}
\argmax_{C\in[0,1]^{m\times n}} \ \|C\|_1 
\ \mathrm{s.t.} \ \frac{1}{\|C\|_1}\sum_{i\in[m], j\in[n]} C_{ij}(M_{ij}^{\star}-\hat{M}_{ij})^2 \le \eps.
\end{align*}
\end{definition}
The matrix $C$ that solves the above optimization problem has the following property: $C_{ij}$ indicates whether the $(i,j)$-th entry shall be completed, and a fractional value can be interpreted as either a fractional completion or a probability of completion. $C$ provides the largest coverage among all that satisfy the property that the average error is at most $\eps$ over the completed entries. 

\paragraph{Challenges of partial matrix completion.}

The above formulation is intuitive and convex, but notice that we do not know $M^\star$, and thus cannot apply convex optimization directly. 

Key challenges are illustrated by
two intuitive, but na\"ive approaches. First, consider completing just those matrix
entries $(i, j)$ where all rank-$k$ completions  have nearly the same
value $\hat{M}_{ij}$. Even in a simple case where all revealed entries are 1,
any single entry could be $\pm 1$ for some rank 2 completion of the matrix. Based on this, one often will not be able to complete any entries whatsoever.
 
Second, consider the generalization of our previous example where the sampling
distribution is uniform over an \textit{arbitrary} subset $U$ of $m \times n$
possible indices. For simplicity, think of $|U|=mn/2$ as covering half the
matrix. It follows from previous generalization bounds on matrix completion that one could accurately complete all entries in $U$. However, in general $U$ is unknown and arbitrarily complex. The second na\"ive approach would be to
try to learn $U$, but this requires
$\approx |U|=mn/2 \gg m+n$ observations.

\subsection{Our contribution}

We present an inefficient algorithm and its efficient relaxation to solve the partial matrix completion problem with provable guarantees over arbitrary sampling distributions. In particular, if $\mu$ is uniform over some arbitrary subset $U\subseteq [m]\times [n]$ of a rank-$k$ matrix $M^\star\in[-1,1]^{m\times n}$, the theoretical results obtained in this work have the following implications. The first, is an inefficient algorithm that solves the partial matrix completion problem for uniform distribution:  
\begin{corollary}[Inefficient algorithm, uniform sampling distribution] \label{corr:inefficient} 
	For any $\eps, \delta > 0$, with probability
$\ge 1-\delta$ over $N \ge c_0 \cdot \frac{k (m + n) + \log(1/\delta)}{\eps^2}$
	revealed entries from $M^\star$ drawn from $\mu$, for some constant $c_0$,
	Algorithm~\ref{alg:inefficient} (Sec.~\ref{sec:coverage}) outputs $\hat{M}$
	and  $C \in [0,1]^{m \times n}$, such that
    \begin{align*}
        (1):& \ \ \  \|C\|_1 \ge |U|,\\
        (2):& \ \ \  \frac{1}{\|C\|_1}\sum_{i\in[m], j\in[n]} C_{ij}(M_{ij}^{\star}-\hat{M}_{ij})^2 \le \eps.
    \end{align*}
\end{corollary}

Next, we describe the guarantees of an efficient algorithm \ref{alg:efficient}. It has a worse sample complexity, up to polynomial factors in $k,\eps$. The advantage is that it runs in polynomial time. Otherwise, its guarantees are similar to that of the first algorithm. 

\begin{corollary}[Efficient algorithm, uniform sampling distribution] \label{corr:efficient}
	For any $\eps, \delta > 0$, with probability
	$\ge 1-\delta$ over $N\ge c_0 \cdot \frac{k^2}{\eps^4} \cdot (k (m + n) + \log
	(1/\delta))$ revealed entries from
	$M^\star$ drawn from $\mu$, for some constant $c_0$, Algorithm~\ref{alg:efficient} (Sec.~\ref{sec:efficient}) outputs
	$\hat{M}$ and $C \in [0, 1]^{n \times m}$, such that
	\begin{align*}
        (1):& \ \ \  \|C\|_1 \ge |U|,\\
        (2):& \ \ \  \frac{1}{\|C\|_1}\sum_{i\in[m], j\in[n]} C_{ij}(M_{ij}^{\star}-\hat{M}_{ij})^2 \le \eps.
    \end{align*}
\end{corollary}

Corollary~\ref{corr:inefficient} is a simplification of Theorem~\ref{thm:main_inefficient}, and Corollary~\ref{corr:efficient}  is a simplification of Theorem~\ref{thm:main_efficient}. 

The results can be extended to the setting where the observed entries
are noisy with zero-mean bounded noise.  A remarkable feature of our
algorithms is that once the full completion $\hat{M}$ is obtained using
existing procedures, they only rely on the \emph{locations} of observed entries
and not the values. Thus, the decisions regarding which entries to complete,
i.e., to add to $C$ is completely agnostic to the actual values of the revealed
entries. Our framework is able to handle \textit{arbitrary} sampling distributions, which captures scenarios such as overlapping groups or idiosyncratic rating habits.

\paragraph{The online  setting.}
We formulate the partial matrix completion  problem in the online learning setting and propose an iterative gradient based online algorithm with provable guarantees. Corollary~\ref{cor:online} is a special case of Theorem~\ref{cor:online-to-offline}. A simple simulation of the online algorithm in this special case of uniform sampling over a subset of entries is demonstrated in Fig.~\ref{fig:my_label}. 

\begin{corollary} 
[Online algorithm, uniform sampling distribution]
\label{cor:online}
Suppose $M^\star\in[-1,1]^{n\times n}$ is a bounded matrix with max-norm (Eq.~\ref{eqn:max-norm-def}) bounded by $K$. The sampling distribution $\mu$ is uniform over a fraction $0<c\le 1$ of the $n^2$ entries. 
For any $\delta>0$, after $T=\tilde{O}(\delta^{-2}K^2n)$ iterations, the confidence matrix $C\in[0,1]^{m\times n}$ output by the online algorithm \texttt{ODD} satisfies with probability at least $1-c_1\exp(-c_2\delta^2T)$ for some universal constants $c_1,c_2>0$,
\begin{align*}
(1):& \ \ \  \|C\|_1 \ge (c-\delta^{1/6})n^2,\\
(2):& \ \ \  \frac{1}{\|C\|_1}\sup_{M\in\mathcal{V}}\sum_{i\in[m], j\in[n]}C_{ij}(M_{ij}-M^{\star}_{ij})^2\le \frac{\delta^{1/6}}{c-\delta^{1/6}},
\end{align*}
where $\mathcal{V}$ is the set of all matrices $M\in[-1,1]^{n\times n}$ with max-norm bounded by $K$ satisfying that $\E_{i,j\sim\mu}[(M_{ij}-M_{ij}^{\star})^2]\le \delta$. 
\end{corollary}

\subsection{Related work}
\paragraph{Matrix completion and recommendation systems.}~%
The common approach to collaborative filtering is that of matrix
completion with trace norm minimization \citep{SrebroThesis}.  It was
shown  that the trace norm regularization requires $\Omega((n+m)^{3/2})$
observations to complete the matrix under arbitrary distributions, even when
the rank of the underlying matrix is constant, and this suffices even in
the more challenging agnostic online learning setting
\citep{HazanKS12,shamir2014matrix}.  \citet{srebro2005rank} study
matrix completion with the max-norm, which behaves better under arbitrary
distributions, namely for low-rank matrices, $\tilde{O}(n+m)$ observations
suffice.

\paragraph{Matrix completion and incoherence assumptions.}~
A line of works in the matrix completion literature considered the goal of finding a completion of the ground truth matrix with low Frobenius norm error guarantee \citep{candes2009exact, candes2010power, keshavan2010matrix, gross2011recovering, recht2011simpler, negahban2012restricted, jain2013low, chen2020noisy, abbe2020entrywise}. Such guarantee is strong but usually requires two restrictive assumptions on sampling distribution of observations and the ground truth matrix structure: (1) the sampling of observations is uniform across all entries, and (2) the ground truth matrix $M^{\star}$ satisfies some incoherence conditions. The incoherence condition is an assumption imposed on the singular vectors of $M^{\star}=U\Sigma V^{\top}$, which ensures the vectors $u_i, v_j$'s to be sufficiently spread out on the unit sphere. The main reason the incoherence condition is necessary in establishing meaningful guarantees in low-rank matrix completion is that without such assumptions, there might exist multiple low-rank matrices that could explain the observed entries equally well but differ substantially on unobserved entries. When the incoherence condition is satisfied, it implies that the observed entries capture enough information about the low-rank structure, making it possible to recover the original matrix with sufficiently well. However, uniform sampling and incoherence assumptions \emph{do not} hold in many realistic scenarios. 

Subsequently, to circumvent these restrictive assumptions, another line of work evaluates the \textit{generalization error} as an alternative metric for completion performance \citep{srebro2004generalization, shamir2014matrix}. To formalize, consider the task of completing $M^{\star}$ with an arbitrary observation sampling distribution $\mu$ over $[m]\times [n]$. This can be conceptualized as predicting a hypothesis mapping from the domain $[m]\times[n]$ to the codomain $[-1,1]$. The goal is to characterize the guarantees on the expected prediction error, quantified for a given completion $\hat{M}$ as $\E_{(i,j)\sim\mu}[(\hat{M}_{ij}-M^{\star}_{ij})^2]$. When $\mu$ is the uniform distribution over $[m]\times [n]$, the generalization error bound translates to a guarantee over the average error across all entries. However, if $\mu$ is arbitrary (e.g. supported over a fraction of the entries), then no guarantee could be established for entries that lie in the complement of the support of $\mu$. This prohibits the use of $\hat{M}$ in settings where abstention is imperative.

Our work takes an alternative approach to this fundamental limitation in matrix completion. The focus of our work is to identify entries where we can predict with high confidence to guarantee a low completion error weighted by evaluated confidence. 
We provide formulation of the problem, its convex relaxation, and an efficient gradient-based online algorithm due to a formulation of the problem as an online two-player game.

\paragraph{Randomized rounding and semi-definite relaxations.}~%
For the efficient algorithm, the main analysis technique we deploy is based on randomized rounding solutions of semi-definite programming, due to the seminal work of \citet{goemans1995improved}. These were originally developed in the context of approximation algorithms for MAX-CUT and other combinatorial problems. Here, we use this analysis technique for analyzing inner products of high dimensional vectors, but in a different context and to argue about a convex relaxation of a continuous optimization problem. 

\paragraph{Abstention in classification and regression.}~%
Abstaining from prediction has a long history in other learning problems such as binary classification \cite{chow1957optimum}. There are several notable similarities. First, many of the algorithms work as ours does, by first fitting a complete classifier and then deciding afterwards where to abstain \citep[e.g.,][]{Hellman70, GKKM:2020}. Second, numerous abstention models have been considered, including a fixed cost for abstaining \citep[e.g.,][]{chow1957optimum, KK21}, as well as fixing the error rate and maximizing the coverage \citep[e.g.,][]{geifman2017selective}, as in our work, and fixing the coverage rate while minimizing the error \citep[e.g.,][]{DBLP:conf/icml/GeifmanE19}. Third, there are algorithmic similarities, in particular the observation that worst-case performance guarantees do not depend on knowing the underlying distribution \citet{KK21}. Online models of abstention have also been considered \citep[e.g.,][]{li2011knows}. For detailed survey on  selective classification  see \citep{ElYanivWiener10}.

\section{Problem Setup and Preliminaries}
\label{sec:prelims}

For a natural number $n$, let $[n] = \{1, 2, \ldots, n\}$. Consider a fixed set
of indices $\inds = [m] \times [n]$ throughout. Thus, in this paper $x \in \inds$ is $x = (i, j)$ where $i \in [m], j
\in [n]$. The set of $m \times n$ real-valued matrices is thus written as
$\reals^\inds$. We will thus view a matrix $M$ over index set $\inds$ as a
function from $\inds \rightarrow \reals$, and for $x \in \inds$, by slight
abuse of notation denote by both $M_x$ and $M(x)$ the entry of $M$ at index
$x$. 

We focus on the squared loss in this paper, though it may be interesting to
consider other loss functions in future work.  Suppose $\nu$ is a distribution
over $\inds$, for matrices, $M, M^\prime$, we denote by $\ell(\nu, M, M^\prime)
\defeq \E_{x \sim \nu}[(M(x) - M^\prime(x))^2]$; we will also use the shorthand
$\norm{M - M^\prime}_\nu^2 \defeq \ell(\nu, M, M^\prime)$. 

When $T \in \inds^N$ is a
sequence of elements of $\inds$ of length $N$, we denote by $\ell(T, M, M^\prime) \defeq \frac{1}{N} \sum_{x \in T} (M(x) - M^\prime(x))^2$ and the
shorthand $\norm{M- M^\prime}_T^2 \defeq \ell(T, M, M^\prime)$. 

\subsection{Matrix Classes, Version Space and Generalization}

\paragraph{Matrix Norms.} 
Given a matrix $M$, we denote by $\norm{M}_{2, \infty}$ the maximum row norm of $M$. We denote by $\mnorm{M}$ the \emph{max-norm} of $M$ defined as,
\begin{align}
\mnorm{M} \defeq \min_{UV^\top = M} \norm{U}_{2, \infty}\cdot \norm{V}_{2, \infty}. \label{eqn:max-norm-def} 
\end{align}
The trace norm of a matrix $M$, denoted by $\trnorm{M}$, is the sum of its
singular values. The Frobenius norm of a matrix $M$, denoted by $\frnorm{M}$, is
the square root of the sum of its squared entries.

We will restrict attention to matrices with entries in $[-1, 1]$. The following three classes of matrices with restrictions on respectively the \emph{rank}, \emph{max-norm} and \emph{trace-norm} are of interest in this work. Formally, let $\tilde{\K} = \{ M \in [-1, 1]^\inds\} $, and define,
\begin{align*}
	\mMrk &\defeq \{ \mathrm{rank}(M) \le k \} \cap \tilde{\K}, \ \ \ \mMmax &\defeq \{ \mnorm{M} \le K \} \cap \tilde{\K}, \ \ \ \mMtr &\defeq \{  \trnorm{M} \le \tau \} \cap \tilde{\K}.
\end{align*}

The standard assumption in matrix completion is that the target matrix is low-rank; however, most completion algorithms exploit convex optimization methods
and optimize over max-norm or trace-norm bounded matrices. It is known that
$\mnorm{M} \le \sqrt{\rank(M)}$ if $M \in [-1, 1]^\inds$~\citep[cf.][Lemma 4.2]{linial2007complexity} and it is easy to see that $\trnorm{M} \le
\rank(M) \sqrt{mn}$ for $M \in [-1, 1]^\inds$. Thus, $\mMrk \subseteq
\mM^{\max}_{\sqrt{k}}$ and $\mMrk \subseteq \mM^{\mathrm{tr}}_{k\sqrt{mn}}$.
Hence, we will work with the latter two classes which satisfy the following lemma which is proved in Appendix~\ref{app:offline}. 

\begin{lemma} \label{lem:convexclosed}
	The classes, $\mMmax$ and $\mMtr$, are closed under negations and are convex.
\end{lemma}

\paragraph{Version Space.}

We define the notion of version spaces that are used in our key results.
 For a sequence $T \in \inds^N$, for any class of matrices $\mM$, a
matrix $M \in \mM$ and $\beta > 0$, we define the version space around $M$ of radius $\beta$ based on $T$ w.r.t. $\mM$ as
\begin{align}
	\V(M, \beta, T; \mM) &= \{M^\prime \in \mM ~|~ \ell(T, M, M^\prime)
  \le \beta \}. \label{eqn:versionspacedef}
\end{align}
Intuitively, version space is the set of matrices in a particular matrix class that are ``close'' to a given matrix with respect to $T$. 

\paragraph{Generalization Bounds.}

For a general class of matrices $\mM$, we define a notion of sample
complexity that will guarantee the proximity of empirical and population
measures of interest for all matrices in the class; we denote this notion by
$\sampc(\eps, \delta, \mM)$.

\begin{definition}[Sample Complexity] \label{defn:samplecomp}
	For $\eps, \delta > 0$ and a class of matrices $\mM$, denote by $\sampc(\eps,
	\delta, \mM)$, the  sample complexity, to be the smallest natural number
	$N_0$, such that for any distribution $\mu$ over $\inds$, for $S \sim
	\mu^{N_0}$, and for any fixed $\hat{M} \in \mM$ (possibly depending on $S$),
	with probability at least $1 - \delta$, 
	\[ \sup_{M \in \mM} \left|\ell(S, M, \hat{M}) - \ell(\mu, M, \hat{M})\right| \le \eps. \]
	If no such $N_0$ exists, $\sampc(\eps, \delta, \mF) \defeq \infty$. 
\end{definition}

Bounds on the sample complexity for matrix completion can be derived in terms of rank,
max-norm and trace-norm using standard results in the literature, and the fact
that the squared loss is $2$-Lipschitz and bounded by $4$ when both its
arguments take values in $[-1, 1]$. In the proposition below, the max-norm
result follows from Theorem 5 in~\citep{srebro2005rank} and the trace-norm
result follows from Theorem 4 in~\citep{shamir2014matrix}. 

\begin{proposition} \label{prop:genbounds}
	For the classes, $\mMmax$, $\mMtr$, the following hold,
	\begin{enumerate}
		\item $\sampc(\eps, \delta, \mMmax) = O\left(\frac{1}{\eps^{2}} \left(K^2  (m + n) +  \log \frac{1}{\delta} \right) \right)$.
			
		\item $\sampc(\eps, \delta, \mMtr) = O\left(\frac{1}{\eps^2}  \left(\tau \sqrt{m + n} + \log \frac{1}{\delta} \right) \right)$.
			
	\end{enumerate}
\end{proposition}

It is worth comparing the two bounds in Proposition~\ref{prop:genbounds}
above. Consider the matrix $M^\star$ consisting of all $1$'s -- from a matrix
completion point of view this is particularly easy as every ``user'' likes
every ``movie''. Note however that $\mnorm{M^\star} = 1$ and $\trnorm{M^\star} =
\sqrt{mn}$. For this example, ignoring the dependence on $\eps, \delta$, the
sample complexity bound obtained using the trace-norm result  would be $O((m +
n)\sqrt{m + n})$, while that using max-norm would be $O(m + n)$. In general, it
is always the case that $\trnorm{M}/\sqrt{mn} \le \mnorm{M}$, so it may seem
that the bounds in terms of trace-norm are weaker. However, there are matrices
for which this gap can be large and the sample complexity bound in terms of
trace-norm is shown to be tight in general (see ~\citep{srebro2005rank,
shamir2014matrix} for further details).

\subsection{Full Completion Problem (with Noise)}

\label{sec:completion-problem}

We now define the matrix completion problem with zero-mean noise, in fact a matrix \textit{estimation} problem in the noisy case. Let $\D$ be a distribution supported on $\inds \times
[-1, 1]$; the results in the paper can all be easily extended when $\D$ is
supported on $\inds \times [-B, B]$, increasing squared error by a $B^2$ factor. Let $\mu$ be the marginal distribution of $\D$ over $\inds$.
Let $S_{XY} = \langle (x_t , y_t) \rangle_{t=1}^N$ be an iid sample drawn from $\D^N$.
Note here that $x_t = (i_t, j_t)$ denotes the index of the matrix and $y_t$ the
observed value. We let $S = (x_1, \ldots, x_N)$ denote the sequence of
$x_t$'s from $S_{XY}$ (with repetitions allowed); note that $S$ is distributed as $\mu^N$. Let $M^\star \in [-1,
1]^\inds$ be the matrix where $M^\star_{ij} = \E_\D[ Y | X = (i, j)]$. We
say that $\hat{M}$ is an $\eps$-accurate completion of $M^\star$, if $\ell(\mu,
M^\star, \hat{M}) \le \eps$. 

\begin{definition}[Full Completion Algorithm] \label{defn:fullcompalg}
	We say that $\FC(S_{XY}, \eps, \delta, \mM)$ is a full completion algorithm
	with sample complexity $\sampc_{\mathsf{FC}}(\eps, \delta, \mM)$ for $\mM$, if
	provided $S_{XY} \sim \D^N$ for some $\D$ over $\inds \times [-1, 1]$ with
	$M^\star \in \mM$, $N \ge \sampc_{\mathsf{FC}}(\eps, \delta, \mM)$, $\FC(S_{XY}, \eps,
	\delta, \mM)$ outputs $\hat{M} \in \mM$ that with probability at least $1 -
	\delta$ satisfies, $\ell(\mu, M^\star, \hat{M})\le \eps$. 
\end{definition}

The following result follows from~\citep{srebro2005rank, shamir2014matrix} (see also Prop.~\ref{prop:genbounds}).
\begin{proposition} \label{prop:fcalgs}
	There exists polynomial time (in $mn/\eps$) full completion algorithms for
	$\mMmax$ and $\mMtr$ (e.g. ERM methods) with the following sample complexity
	bounds:
	\begin{enumerate}
		\item 
			$\sampc_{\mathsf{FC}}(\eps, \delta, \mMmax) = O\left(\frac{1}{\eps^{2}} \left(K^2  (m + n) +  \log \frac{1}{\delta} \right) \right)$.
		\item 
			$\sampc_{\mMtr}(\eps, \delta, \mMtr) = O\left(\frac{1}{\eps^2}  \left(\tau \sqrt{m + n} + \log \frac{1}{\delta} \right) \right)$.
	\end{enumerate}
\end{proposition}

\subsection{The Partial Matrix Completion Problem}
\label{sec:coverage-problem}
Recall the definition of the (fractional) partial matrix completion in Eq.~\ref{eq:nonconvex-pmc} and Def.~\ref{def:fractional-pmc}. In this work, we will assume that a full completion matrix $\hat{M}$ is already obtained. The focus of the present work is finding a completion matrix $C \in \{0, 1\}^\inds$ (or $C \in [0,1]^\inds$ for fractional coverage) with large \textit{coverage}.  The coverage in both cases are defined as: $|C|\defeq \sum_{x \in \inds} C_{x},$
and a low \textit{loss}, defined as: 
$$\ell(C, M^\star, \hat{M}) \defeq \frac{1}{|C|} \sum_{x \in \inds} C_{x} (M^\star_{x}-\hat{M}_{x})^2.$$

Note that $C \in [0, 1]^\inds$ can be viewed as a ``fractional'' set and we are overloading the notation $\ell$ to allow that. Such a fractional coverage can be randomly rounded to a set whose size is within 1 of $|C|$. In an ideal world, we would have at most $\eps$ loss measured over the cells we complete as defined above, and large or even full coverage $|C| = mn$.

Note that although the requirement in the optimization problems in Eq.~(\ref{eq:nonconvex-pmc}) and Def.~\ref{def:fractional-pmc} is to guarantee $\ell(C, M^\star, M)
\le \eps$, as we don't know $M^\star$, the only guarantee we have by using a
full completion algorithm (and generalization bounds) is that (with high probability) $M^\star$
is in some version space, say $\V$, centered at $\hat{M}$. So we actually show
a stronger guarantee that, $\sup_{M \in \V} \ell(C, M, \hat{M}) \le \eps$. 

A second observation is that if $M$ is a version space around $\hat{M}$, then by
convexity of the matrix classes, the matrix $(M - \hat{M})/2$, must be in some
version space, say $\V_0$, centered around the zero matrix, $\bzero_\inds$.
Thus, we will actually find a $C$ that guarantees, $\ell(C, M, \bzero_\inds)
\le \eps$ for every $M \in \V_0$. This means that for maximizing coverage our
algorithm has the remarkable property we only need to know the locations of the
revealed entries as indicated by $S$. 

Section~\ref{sec:coverage} presents a computationally inefficient but
statistically superior algorithm in terms of sample complexity for the \emph{coverage
problem} and its consequences for partial matrix completion.
Section~\ref{sec:efficient} presents a computationally efficient algorithm at
the cost of a slightly worse sample complexity when using the class $\mMmax$.

\section{An Inefficient Algorithm}
\label{sec:coverage}

The main novelty  in  the partial matrix completion problem is that of finding an  optimal coverage, defined as the matrix $C$  in  the formulation in Def.~\ref{def:fractional-pmc}.  In this  section  we give an (inefficient) algorithm for finding the optimal coverage, and the generalization error bounds  for Partial Matrix Completion that arise  from it. In the next section we  give an efficient approximation algorithm for doing the same, with slightly worse sample complexity bounds and more complex analysis. This differs substantially from prior work on abstention in classification and regression \citep{KK21} where the optimal solution can be found in polynomial time.

Let $C \in [0, 1]^\X$ be a target confidence matrix. For such a $C$, we will
denote by $\nu_C$ the probability distribution where $\nu_C(x) = C_x/ \|C\|_1$ for
$x \in \inds$. Let $S \sim \mu^N$ be a sample obtained from the target
distribution $\mu$. We output a $C$ which is an optimal solution to the
following optimization problem, inspired by a similar problem studied
by \citet{KK21} for classification with abstention.

\fbox{
	\parbox{0.97\columnwidth}{ 
		\begin{align}
        & \text{Parameters and inputs}: \gamma, \beta, S, \mM \notag\\
        & \text{maximize} \ \ \ \|C\|_1  \label{opt:inefficient} \\
        & \text{subject to } C \in [0,1]^{\X}, \ \text{and} \quad \forall M \in \V(\bzero_\inds, \beta, S; \mM), \  \ell(\nu_C, M, \bzero_\inds)  \le \gamma  \notag
        \end{align}
    }
}

\begin{algorithm}[h!]
\begin{algorithmic}[1]
\STATE \textbf{Inputs}: $S_{XY} = \langle (x_t, y_t) \rangle_{t=1}^N \sim \D^N$, $\eps$, $\delta$, $\mM$, $\FC$ (cf. Def.~\ref{defn:fullcompalg}).
\STATE Obtain $\hat{M} \in \mM$ using $\FC(S_{XY}, \eps/4, \delta/3, \mM)$.
\STATE Obtain $C$ using MP~\ref{opt:inefficient} with $\gamma \defeq \eps/4$, $\beta \defeq \eps/8$, $S = (x_1, x_2, \ldots, x_N)$ from $S_{XY}$.
\STATE \textbf{return} $(\hat{M}, C)$.
\caption{\label{alg:inefficient}}
\end{algorithmic}
\end{algorithm}

MP~\ref{opt:inefficient} defines a family of optimization problems; we will
typically use $\mM$ to be $\mMmax$ or $\mMtr$. The above optimization problem is
in fact a linear program. The only unknowns are $C_x$ for $x \in \inds$.
However, the set of constraints is infinite and it is unclear how a separation
oracle for the constraint set may be designed. Such an oracle could be designed
by solving the following optimization problem: 
\begin{equation}\label{eqn:hard-separation-problem}%
\max_{M \in \V(\bzero_\inds, \beta, S; \mM)} \ell(\nu_C, M, \bzero_\inds). 
\end{equation}

This optimization problem requires maximizing a quadratic function, and we prove it to be 
$\NP$-hard with the additional PSD and symmetric constraints in Appendix~\ref{sec:hardness}.

The following result subsumes Corollary~\ref{corr:inefficient}. 
We defer the proof of this theorem to Appendix~\ref{app:offline}, as the proof is similar to that of our efficient algorithm in the next section.

\begin{theorem} \label{thm:main_inefficient}
	Let $\mM$ be either $\mMmax$ or $\mMtr$. Let $\D$ be distribution over $\inds
	\times [-1, 1]$, $\mu$ the marginal of $\D$ over $\inds$, $\mu_{\max}=\max_{i,j}\mathbb{P}_{\mu}((i,j)\text{ is sampled})$. Suppose that
	$M^\star$, defined as $M^\star_{ij} = \E_{(X, Y)\sim \D}[Y | X = (i, j)]$
	satisfies that $M^\star \in \mM$. Furthermore, suppose that $S_{XY}
	\sim \D^{N}$ and that $\FC$ is a full completion algorithm as in Defn.~\ref{defn:fullcompalg}. Then, provided $N \ge \max\{\sampc_{\mathrm{FC}}(\eps/4, \delta/3, \mM),
	\sampc(\eps/8, \delta/3, \mM) \}$, for $(\hat{M}, C)$ output by
	Alg.~\ref{alg:inefficient}, it holds that:
    \begin{enumerate}
        \item $\|C\|_1 \ge 1/\mu_{\max}$,
        \item $\displaystyle\frac{1}{\|C\|_1} \sum_{x \in \inds} C_x (\hat{M}_x - M^\star_x)^2 \le \eps.$
    \end{enumerate}
\end{theorem}

\section{An Efficient Algorithm} \label{sec:efficient}

In this section, we show how the result from
Proposition~\ref{prop:inefficient} can be achieved using an efficient
algorithm at a modest cost in terms of sample complexity when using the matrix
class $\mMmax$. In particular, we consider the optimization problem defined in
MP~\ref{opt:efficient}, where the constraint $\norm{M}_{\nu_C}^2 \le \gamma$
for all $M \in \V(\bzero_\inds, \beta, S; \mMmax)$ from
MP~\ref{opt:inefficient} is replaced by $\E_{x \sim \nu_C}[M_x] \le \gamma$.
The efficiency comes from the fact that we can now implement a separation
oracle for the constraint set by solving, for a given $C$, the problem, 
\[ \max_{M \in \V(\bzero_\inds, \beta, S, \mMmax)} \E_{x \sim \nu_C}[M_x]. \]
Mathematical program \ref{opt:efficient} is a convex optimization problem since the constraint set is convex and the objective function is linear. However, it is not immediately clear how this relaxed optimization problem relates to the original, which is the main technical contribution of this section.

\fbox{
	\parbox{0.97\columnwidth}{ 
		\begin{align} \label{opt:efficient}
        & \text{Parameters and inputs}: \gamma, \beta, S, \mMmax \notag\\
        & \text{maximize} \ \ \ \|C\|_1 \\
        & \text{subject to } C \in [0,1]^{\X}, \ \text{and } \forall M \in \V(\bzero_\inds, \beta, S; \mMmax),  \ \E_{x \sim \nu_C}[M_x]  \le \gamma \notag 
        \end{align}
    }
}

\begin{algorithm}[h!]
\caption{(Efficient) Offline Algorithm for Partial Matrix Completion}
\begin{algorithmic}[1]
\STATE \textbf{Inputs:} $S_{XY} = \langle (x_t, y_t) \rangle_{t=1}^N \sim \D^N$, $\FC$ (cf. Def.~\ref{defn:fullcompalg}), $\eps$, $\delta$.
\STATE Obtain $\hat{M} \in \mMmax$ using $\FC(S_{XY}, \eps^2/(4 \pi^2 K^2), \delta/3, \mMmax)$.
\STATE Obtain $C$ using MP~\ref{opt:efficient} with $\gamma \defeq \eps/(2 \pi K)$, $\beta \defeq \eps^2/(8 \pi^2 K^2)$, $S = (x_1, x_2, \ldots, x_N)$ from $S_{XY}$.
\STATE \textbf{return} $(\hat{M}, C)$.
\label{alg:efficient}
 \end{algorithmic}
\end{algorithm}

The following result subsumes Corollary~\ref{corr:efficient}. 

\begin{theorem} \label{thm:main_efficient}
	Let $\D$ be a distribution over $\inds
	\times [-1, 1]$, $\mu$ the marginal of $\D$ over $\inds$, $\mu_{\max}=\max_{i,j}\mathbb{P}_{\mu}((i,j)\text{ is sampled})$. Suppose that
	$M^\star$, defined as $M^\star_{ij} = \E_{(X, Y)\sim \D}[Y | X = (i, j)]$
	satisfies that $M^\star \in \mMmax$. Furthermore, suppose that $S_{XY}
	\sim \D^{N}$ and that $\FC$ is a full completion algorithm as in Defn.~\ref{defn:fullcompalg}.
	Then, provided $N \ge \max\{\sampc_{\mathrm{FC}}(\eps^2/(4 \pi^2 K^2), \delta/3, \mMmax),$ $\sampc(\eps^2/(8 \pi^2 K^2), \delta/3, \mMmax) \}$, for $(\hat{M}, C)$ output by
	Alg.~\ref{alg:efficient}, it holds that: 
    \begin{enumerate}
        \item $\|C\|_1 \ge 1/\mu_{\max}$,
        \item $\displaystyle\frac{1}{\|C\|_1} \sum_{x \in \inds} C_x (\hat{M}_x - M^\star_x)^2 \le \eps$. 
    \end{enumerate}
\end{theorem}

\section{Online Setting}

The previous section introduces the \textit{offline} setting for 
the partial matrix completion problem. In Appendix \ref{sec:online}, we describe the online version of the problem, which is motivated by two important considerations. First, in many applications, the observation pattern is more general than a fixed distribution. It can be a changing distribution or be comprised of adversarial observations. Second, our online algorithm incrementally updates the solution via iterative gradient methods, which is more efficient than the offline methods. For space considerations, details are deferred to Appendix~\ref{sec:online} which contains the setting, definitions, algorithm specification, and main results, and Appendix~\ref{app:online}, which details the proofs. In particular, the online algorithm, called Online Dual Descent (ODD), is described in Algorithm \ref{alg:odd}, and its regret guarantee in Theorem \ref{sec:MainResult}. The online regret guarantee implies the statistical learning guarantees of the previous sections when the support size is a constant fraction of the full matrix, and this implication is spelled out precisely in Corollary \ref{cor:online-to-offline}.

\section{Experiments and Implementation}
\label{sec:experiments}
The MovieLens dataset (\citep{10.1145/2827872}) consists of, among other data, a set of users with their rankings of a set of movies. It is a common benchmark for matrix completion because some of the rankings are missing, and one can make predictions on the rankings of user preferences. 

We used the dataset differently, aiming to test our online algorithm on generating a satisfying confidence matrix. The experimental procedure is outlined as follows: we used training data from $250$ users and their ratings on $250$ movies, giving us a total of $5189$ completed samples from the incomplete matrix of size $250\times 250$. We ran our algorithm, \texttt{ODD} (Algorithm~\ref{alg:odd}), to get a confidence matrix $C$. In parallel, we used another standard matrix completion tool, \texttt{fancyimpute} (\citep{fancyimpute}), to fill in the missing entries of the matrix and obtain a completion $\hat{M}_{f}$. After $C$ and $\hat{M}_{f}$ are obtained, we reveal the true ratings at the missing entries using the validation set and computed the mean squared error of the predicted rating and true rating at each entry. The following plots show the distribution of the mean squared error with respect to the confidence score $C$ assign at the particular entry. 

\begin{figure}[h]
\begin{subfigure}{0.48\textwidth}
\includegraphics[scale=0.35]{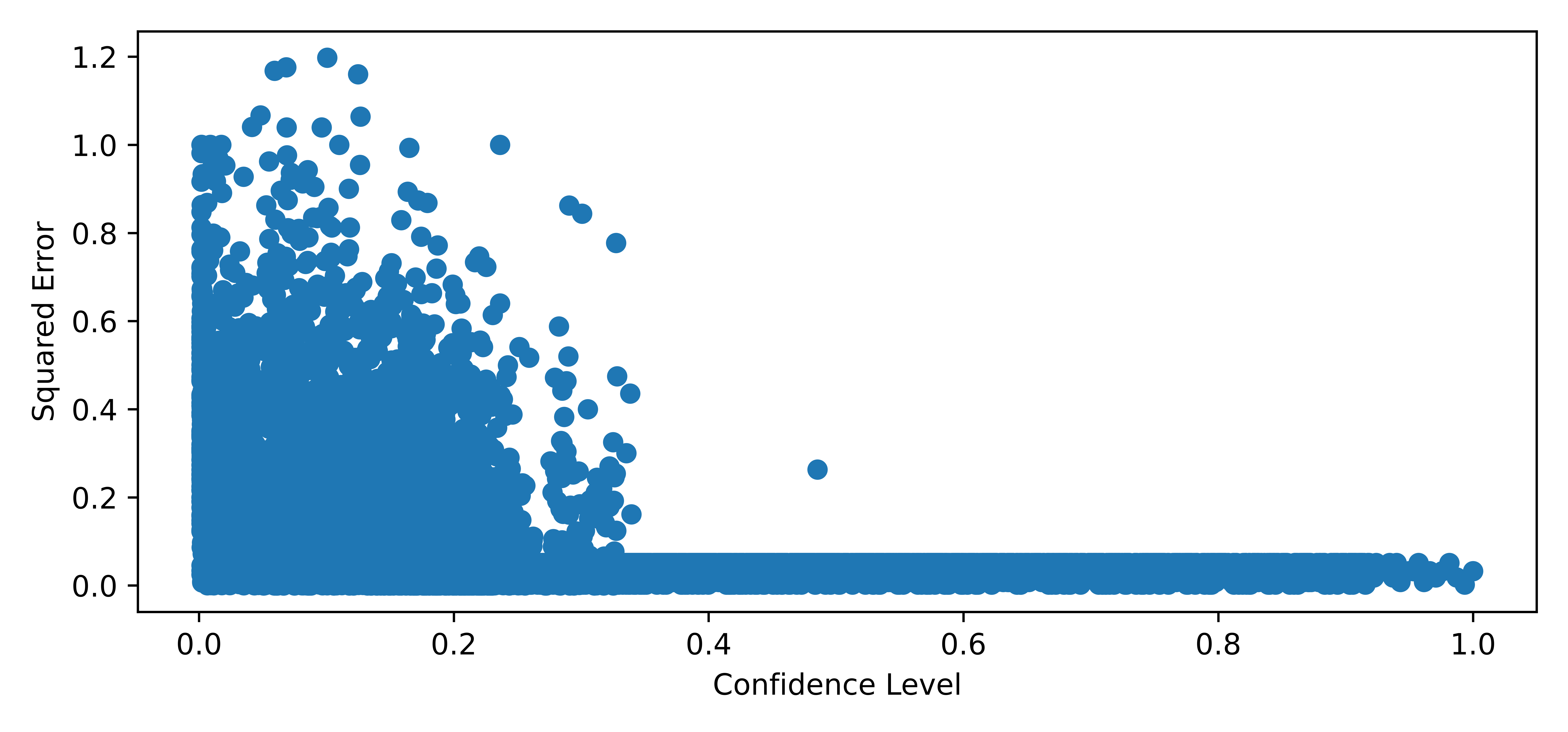}
\end{subfigure}
\begin{subfigure}{0.3\textwidth}
\includegraphics[scale=0.35]{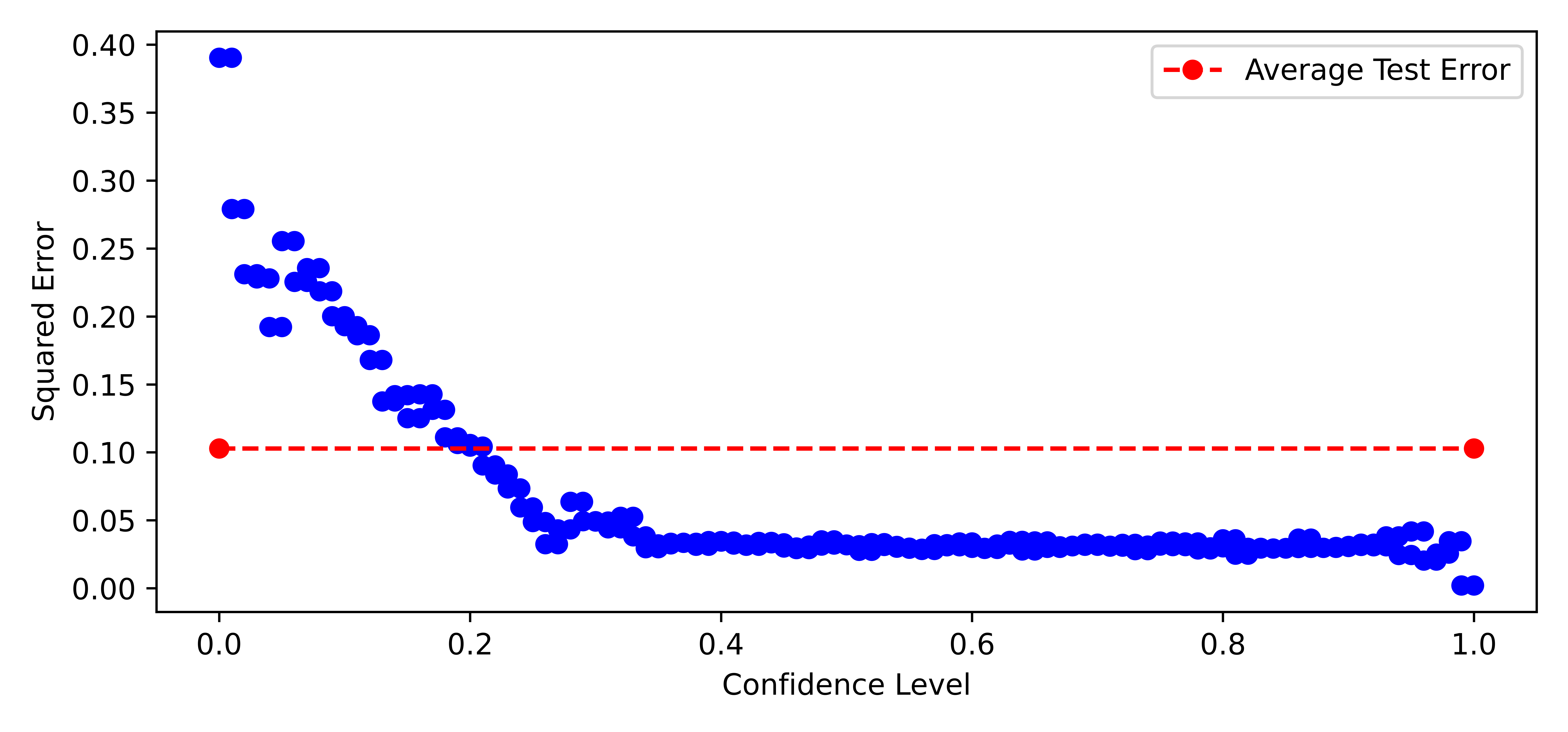}
\end{subfigure}%

\caption{(a) \texttt{ODD} heuristic on $250\times 250$ user-movie rating data from the MovieLens Dataset. 
Squared error in entry-wise prediction and entry-wise confidence level estimated by the \texttt{ODD} heuristic. 
(b) The average squared error over various confidence levels suggests that a higher confidence level is correlated with a lower squared error. }

\end{figure}

\section{Conclusion}

In this work we define the setting of partial matrix completion, with our high-level contributions outlined as the following: 

\paragraph{A new framework.} We propose a new framework called \textit{partial matrix completion}. In this problem, one is given possibly noisy observations from a low-rank matrix, and the goal is to identify entries that can be predicted with high confidence. The partial matrix completion problem answers two key questions: 

\begin{itemize}
	\item Based on possibly noisy samples from an arbitrary unknown sampling distribution $\mu$, \textit{how many entries can be completed} with $\le \eps$ error? \vspace{-0.5em}
    \item How can we efficiently identify the entries that can be accurately completed?
\end{itemize}
When the underlying matrix has low rank $k$, we show that it is possible to complete a large fraction of the matrix using only $\tilde{O}(k (m + n))$ observations while simultaneously guaranteeing high accuracy over the completed set. 
We then study the complexity of identifying the optimal completion matrix. We show that a na\"{i}ve mathematical programming formulation of the problem is hard. However, we propose a relaxation that gives rise to efficient algorithms, which results in a slightly worse dependence on $k$ for the sample complexity. These guarantees are outlined both in Corollary~\ref{corr:inefficient} and Corollary~\ref{corr:efficient}, and Theorem~\ref{thm:main_inefficient} and Theorem~\ref{thm:main_efficient} for their more general versions. 

\paragraph{Online game formulation.} Furthermore, we consider the partial matrix completion problem in the online setting, where the revealed observations are \textit{not} required to follow any particular fixed distribution. The goal henceforth is to minimize \textit{regret}, the gap between the algorithm's performance and the single best decision in hindsight. This is a more general setting, as when imposed with distribution assumptions, regret guarantees in online learning algorithms naturally translate to statistical guarantees. 

Our proposed online partial matrix completion algorithm is derived from an online repeated game. The version space is the set of all valid completions with low generalization error.  High confidence should be assigned to entries where all completions in the version space are similar. Therefore, we formulate the problem as a two-player online game. One player iteratively updates a confidence matrix, and the other learns the version space. We gave an iterative gradient-based method with provable regret guarantees and concluded with preliminary experimental evidence of the validity of our framework.

\section{Acknowledgements}
Elad Hazan acknowledges funding from the Office of Naval Research  grant N000142312156, the NSF award 2134040, and Open Philanthropy. This work was done in part when Clara Mohri was visiting and supported by  Princeton University. 

\bibliographystyle{plainnat}
\bibliography{refs}

\newpage

\appendix

\tableofcontents
\section{Supporting Proofs}
\label{app:offline}

\subsection{Proof of Lemma \ref{lem:convexclosed}
\label{app:convexclosedProof}}
\begin{replemma}{lem:convexclosed}
	The classes, $\mMmax$ and $\mMtr$, are closed under negations and are convex.
\end{replemma}
\begin{proof}[Proof of Lemma~\ref{lem:convexclosed}]
Suppose $M$ is a matrix then we can write $M = U V^\top$ such that,
	\begin{itemize}
		\item If $M \in \mMmax$, $\norm{U}_{2, \infty} \cdot \norm{V}_{2, \infty} \le K$. 
		\item If $M \in \mMtr$, $\frnorm{U} \cdot \frnorm{V} \le \tau$. 
	\end{itemize}

	Note that $-M = -U V^\top$, negating the sign of $U$ doesn't affect its norm, so clearly the classes are closed under negation.

    The convexity of the class follows directly from that $\|\cdot\|_{\max}$ and $\|\cdot\|_{\trace}$ are well-defined norms.

	%
	%
\end{proof}

\subsection{Computational Hardness of the Coverage Problem}\label{sec:hardness}
In this section we show that the optimization problem MP~\ref{opt:inefficient}
is NP-hard in full generality. The first fact we use is the polynomial time
equivalence between linear optimization and separation that was established in
the work of \cite{grotschel1981ellipsoid}. It thus remains to prove NP-hardness
of the separation problem, which is given in equation \ref{eqn:hard-separation-problem}, and can be restated as the following mathematical program when our version space is restricted with a trace norm constraint:
\begin{align} \label{eqn:opt} 
& \max_{X} \sum_{ij} C_{ij } X_{ij}^2  \\
&  |X_{ij}| \le 1 \notag , \trnorm{X} \le k \notag . 
\end{align}
We give evidence of NP-hardness for this program by considering the same mathematical program added a symmetric positive-definite constraint, as follows.
\begin{lemma}
Mathematical program \eqref{eqn:opt-psd} is NP-hard to compute, or approximate with factor $k^{1-\eps}$ for any $\eps > 0$.  
\begin{align} \label{eqn:opt-psd} 
& \max_{X} \sum_{ij} C_{ij } X_{ij}^2  \\
& 0 \leq X_{ij} \le 1 \notag ,  X \succeq 0 \ , X\in \mathrm{Sym}(n), \ \trace(X) \le k \notag . 
\end{align}
\end{lemma}


\begin{proof}
The proof relies on strong hardness of approximation results for the MAX-CLIQUE problem that were proven in \cite{hastad1996clique} and subsequent work. 

We prove by reduction from $k$-CLIQUE. Let $G(V,E)$ be an instance of the $k$-CLIQUE problem.

{\bf The reduction.} Given a graph $G$, let 
$$ \reals^{V \times V} \ni C_{ij} = \mycases {1}{$(i,j) \in E$ \text{ or } $i = j$}{0}{otherwise} .$$ 

\end{proof}

We now claim the following:
\begin{lemma}
The value of mathematical program \eqref{eqn:opt} is at least ${k^2}$  if $G$ contains clique of size $k$. Conversely, if the value of \eqref{eqn:opt} is at least $k^2$, then $G$ contains a clique of size at least  $k$. 
\end{lemma}
\begin{proof}

{\bf Completeness} 
If $G$ has a $k$ clique, then consider the following solution $X$. Let
$$ v_i = \mycases{1}{$i \in$ clique}{0}{o/w} $$
and let $X = vv^\top$. Then we have that $\trace{X}=k$, $X_{ij} \in \{0,1\}$. In addition, by definition of $X$, we have that 
$$ \sum_{ij} C_{ij} X_{ij}^2 = \sum_{ij} C_{ij} v_{i}^2 v_j^2 = k^2 .$$

{\bf Soundness} 
Suppose program \eqref{eqn:opt-psd} has a solution $X$ with value $k^2$ and trace exactly $k$. The trace equality is w.l.o.g since we can always increase the diagonal entries while preserving constraints and only increasing the objective.  Therefore, without loss of generality, we may assume that $\trace(X)=k$. 

Next, we claim that w.l.o.g. we have that $\rank(X) = 1$. To see this, notice that we can define a vector $u_i = \sqrt{X_{ii}}$, and $\tilde{X} = uu^\top$. Now, we have that $\trace(\tilde{X}) = k$, and it satisfies the bounded-ness constraints by definition. In addition, we have that the objective is only increased, since 
$$ \sum_{ij} C_{ij} (v_i v_j)^2 = \sum_{ij} C_{ij} {X_{ii}} {X_{jj}} \geq  \sum_{ij} C_{ij} X_{ij}^2 , $$
where the last inequality is by positive semi-definiteness. 

Thus, we can restrict our attention to the set of solutions given by  $\K_k = \{ uu^\top  \ , \ 0 \leq u_{i} \leq 1 \ ,  \ \|u\|_1  = k \}$. We claim that $\K_k$ can be alternatively characterized as the convex hull of all rank-one matrices of the following form:
$$ \K_k = \conv \left\{ vv^\top | v \in \{0,1\}^V \ , \ \|v\|_1 = k \right\} .$$
This fact is shown in page 279 of \cite{warmuth2010blessing}.

We can now continue with the soundness proof. Since the objective $\sum_{ij} C_{ij} X_{ij}^2 $ is a convex function, given a distribution over points in $\K_k$, the maximum is obtained in a vertex. Thus, there exists a binary vector $v$ such that its trace is $k$, and for which $ \sum_{ij} C_{ij} v_i v_j = k^2$. 

Define a subgraph according to $v$ in the natural way: $i \in S$ if and only if $v_i  = 1$. Notice that the subset of vertices $S$ is of size $k$ due to the trace. 

In terms of number of edges in this subgraph, notice that 
\begin{eqnarray*}
|E(S)| & = \frac{1}{2}  \sum_{ij \in S} \bone_{(i,j) \in E} =  \frac{1}{2}\sum_{ij} C_{ij}  v_i v_j =   \frac{k^2}{2} . 
\end{eqnarray*}

\ignore{ 
Moreover,
\begin{align*}
\E[|E(S)|^2]&=\frac{1}{4}\E\left[\sum_{ij\in S}\mathbf{1}_{(i,j)\in E}+\sum_{ij\neq kl, ij, kl\in E}\mathbf{1}_{(i,j)\in E}\mathbf{1}_{(k,l)\in E}\right]\\
&\le \frac{1}{4}\E\left[\sum_{ij}X_{ii}X_{jj}+\sum_{ij\neq kl}X_{ii}X_{jj}X_{kk}X_{ll}\right]\\
&\le \frac{k^2}{4}+\frac{k^4}{4}
\end{align*}
and thus $\mathrm{Var}(|E(S)|)\le \frac{k^2}{4}$. By Chebyshev's inequality, with probability at least $\frac{3}{4}$, $|E(S)|\ge \frac{k^2}{2}-k$. 

Thus, with probability at least $\frac{1}{2}$, both events hold. We have found a clique of size $4k$. 
}
Thus, we have found a clique of size $k$. 
\end{proof}

We note that, although we have shown NP-hardness for the trace-norm with symmetric PSD constraints, it is possible that the optimization problem is efficiently solvable for the max-norm or rank.

\ignore{
\subsection{Offline Inefficient Algorithm Specification}
\label{app:inefficient_alg}
The following algorithm is used in Theorem \ref{thm:main_inefficient}, to solve optimization problem \ref{opt:inefficient} with the appropriate parameters.
\begin{algorithm}
\caption{(Inefficient) Offline Algorithm for Partial Matrix Completion}
\label{alg:inefficient}
\begin{algorithmic}[1]
\STATE \textbf{Input:} $S = \{ (x_t, y_t) \}_{t=1}^N \sim \D^N$, $\FC$ (cf. Defn.~\ref{defn:fullcompalg}), $\eps$, $\delta$, $\mM$
\STATE Obtain $\hat{M} \in \mM$ using $\FC(S, \eps/4, \delta/3, \mM)$.
\STATE Obtain $C$ using MP~\ref{opt:inefficient} with $\gamma \defeq \eps/4$, $\beta \defeq \eps/8$, $S = (x_1, x_2, \ldots, x_N)$ from $S_{XY}$. 
\STATE \textbf{return} $(\hat{M}, C)$.
\end{algorithmic}
\end{algorithm}

}

\subsection{Proof of Theorem \ref{thm:main_inefficient}}
\label{app:main_inefficient_proof}
Recall our main theorem for the inefficient algorithm, which we prove in this appendix. 

\begin{reptheorem}{thm:main_inefficient}
	Let $\mM$ be either $\mMmax$ or $\mMtr$. Let $\D$ be distribution over $\X
	\times [-1, 1]$, $\mu$ the marginal of $\D$ over $\X$. Suppose that
	$M^\star$, defined as $M^\star_{ij} = \E_{D}[Y | X = (i, j)]$
	satisfies that $M^\star \in \mM$. Furthermore, suppose that $S
	\sim \D^{N}$ and that $\FC$ is a full completion algorithm as in Defn.~\ref{defn:fullcompalg}. Then, provided $N \ge \max\{\sampc_{\mathrm{FC}}(\eps/4, \delta/3, \mM),
	\sampc(\eps/8, \delta/3, \mM) \}$, for $(\hat{M}, C)$ output by
	Alg.~\ref{alg:inefficient}, with probability at least $1-\delta$, it holds that:
	\begin{enumerate}
		\item $\|C\|_1 \ge 1/\mu_{\max}$ and
		\item $\displaystyle\frac{1}{\|C\|_1}\sum_{x \in \X} C_x
        (\hat{M}_x - M^\star_x)^2 \le \eps.$
	\end{enumerate}
\end{reptheorem}

The proof relies on the following proposition.

\begin{proposition} \label{prop:inefficient}
	Suppose $S \sim \mu^N$ for some distribution $\mu$ over $\X$; let
	$\mu_{\max} \defeq \max_{x \in \X} \mu_x$. Suppose $\mM$ is one of
	$\mMmax$ or $\mMtr$, let $N \ge \sampc(\gamma/2, \delta, \mM)$ (cf.
	Defn.~\ref{defn:samplecomp}) and suppose $C$ is the maximizer of
	MP~\ref{opt:inefficient} with $\beta = \gamma/2$, then with probability at
	least $1 - \delta$, we have $\|C\|_1\ge 1/\mu_{\max}$.
\end{proposition}
\begin{proof}
	Based on the condition on $N$ and using $\hat{M} = \bzero_\inds$ in
	Defn.~\ref{defn:samplecomp}, with probability at least $1 - \delta$, the
	following holds for all $M \in \mM$:
	\begin{align}
		\left| \norm{M}_\mu^2 - \norm{M}_S^2 \right| &\le \gamma/2 \label{eqn:apply-zero} 
	\end{align}

	We need to show that some $C^\mu$, with $\|C^\mu\|_1 \ge 1/\mu_{\max}$ is
	feasible. Let $C^\mu_x = \mu(x)/\mu_{\max}$. Then, clearly $C_x \in [0, 1]$,
	$\|C^\mu\|_1 = 1/\mu_{\max}$ and $\nu_{C^\mu} = \mu$. Using
	Eq.~\eqref{eqn:apply-zero}, we have that, $\norm{M}_\mu^2 \le \norm{M}_S^2
	+ \gamma/2$, and as a result for every $M \in \V(\bzero_\inds, \beta, S;
	\mM)$, it holds that $\norm{M}_\mu^2 \le \beta + \gamma/2 = \gamma$. Thus, $C_\mu$ is a
	feasible solution to the optimization problem with objective value at least
	$1/\mu_{\max}$. 
\end{proof}

\begin{proof}[Proof of Theorem~\ref{thm:main_inefficient}]
	There are three bad events, for each of which we bound their probability by at
	most $\delta/3$. First, $\FC$ succeeds with probability at least $1 - \delta/3$. 

	Second, provided $N \ge \sampc(\eps/8, \delta/3, \mM)$, for $\hat{M}$ output
	by $\FC$, using Defn.~\ref{defn:samplecomp} it must hold
	for each $M \in \mM$ that with probability at least $1 - \delta/3$,
	\begin{align}
		\left| \norm{M - \hat{M}}_\mu^2 - \norm{M - \hat{M}}_S^2 \right| &\le \eps/8 \label{eqn:apply-mhat} 
	\end{align}
	In particular, when $\beta = \eps/8$, this means that (assuming $\FC$ has not failed),
	\begin{align*}
		\norm{(M^\star - \hat{M})/{2}}^2_S &\le \frac{1}{4} \norm{M^\star
		- M}_\mu^2 + \frac{\epsilon}{32} \le \frac{\eps}{16} + \frac{\eps}{32} \le \beta.
	\end{align*}
	As a result, and also using the convexity of $\mM$ (cf.
	Lemma~\ref{lem:convexclosed}), $(M^\star - \hat{M})/2 \in \V(\bzero_\inds,
	\beta, S; \mM)$.

	Finally, for $N \ge \sampc(\eps/8 = \gamma/2, \delta/3, \mM)$, 
	Proposition~\ref{prop:inefficient} guarantees with probability at least $1
	- \delta/3$, that the output of MP~\ref{opt:inefficient} satisfies
	$\|C\|_1 \ge 1/\mu_{\max}$. 

	The constraint of MP~\ref{opt:inefficient} together with the fact that $(M^\star - \hat{M})/2 \in \V(\bzero_\inds, \beta, S; \mM)$ guarantees that, $\norm{M^\star - \hat{M}}_{\nu_C}^2 \le 4 \gamma$ which completes the required proof.
\end{proof}

\subsection{Proof of Theorem \ref{thm:main_efficient}}

\begin{proposition} \label{prop:efficient}
	Suppose $S \sim \mu^N$ for some distribution $\mu$ over $\inds$; let
	$\mu_{\max} \defeq \max_{x \in \inds} \mu_x$. Let $N \ge \sampc(\gamma^2/2,
	\delta, \mMmax)$ (cf. Defn.~\ref{defn:samplecomp}) and suppose $C$ is the
	maximizer of MP~\ref{opt:efficient} with $\beta = \gamma^2/2$, then with
	probability at least $1 - \delta$, we have $\|C\|_1 \ge 1/\mu_{\max}$.
\end{proposition}

\begin{proof} [Proof of Proposition~\ref{prop:efficient}]
	Based on the condition on $N$ and using $\hat{M} = \bzero_\inds$ in Defn.~\ref{defn:samplecomp} with probability at least $1 - \delta$ it holds for every $M \in \mMmax$:
	\begin{align}
		\left| \norm{M}_\mu^2 - \norm{M}_S^2 \right| \le \gamma^2/2. \label{eqn:applyzeroefficient}
	\end{align}
	As in the proof of Proposition~\ref{prop:inefficient}, let $C^\mu_x =
	\mu_x/\mu_{\max}$ and notice that $C^\mu_x \in [0, 1]$, $\|C^{\mu}\|_1 = 1/\mu_{\max}$ and  $\nu_{C^\mu} = \mu$. Using
	Eq.~\eqref{eqn:applyzeroefficient}, for any $M \in \V(\bzero_\inds, \beta,
	S; \mMmax)$, we have that $\norm{M}_\mu^2 \le \beta + \gamma^2/2 \le
	\gamma^2$. Finally, using the fact that $\E_{x \sim \mu}[M_x] \le
	\sqrt{\norm{M}_\mu^2}$, we get that $C^\mu$ is feasible and the result
	follows.
\end{proof}

The key to proving Theorem~\ref{thm:main_efficient} is the following proposition, whose proof contains the main technical innovation. The derivation of the main theorem given the proposition is in the Appendix.

\begin{proposition} \label{prop:key-GW}
	For the class $\mMmax$, for any distribution  $\nu$ over $\inds$, $\beta > 0$ and $S \in \inds^N$, the following holds: for $\V=\V(\bzero_\inds, \beta, S; \mMmax)$,	%
	\[ \sup_{M \in \V} \ell(\nu, M, \bzero_\inds) \le  \frac{\pi K}{2} \sup_{M \in \V} \E_{x \sim \nu}[M_x]. \]
\end{proposition}

\begin{proof} [Proof of Proposition~\ref{prop:key-GW}]

For succinctness let $\V_0 = \V(\bzero_\inds, \beta; S, \mMmax)$. Let $\nu$ be a an arbitrary distribution over $\inds$ and since we know that $\V_0$ is closed and bounded, we know that a matrix $M \in \V_0$ exists which achieves the value of $\sup_{M \in \V_0} \ell(\nu, M, \bzero_\inds)$. We will use the probabilistic method to show that there exists $\tilde{M} \in \V_0$ such that $\E_{x \sim \nu}[\tilde{M}_x] \ge \frac{2}{\pi K} \ell(\nu, M, \bzero_\inds)$. 

Since $M \in \mMmax$, there exist $U, V$ such that $M = UV^\top$ and that $\norm{U}_{2, \infty} \cdot \norm{V}_{2, \infty} \le K$.
Denote by $u_i, v_j$ respectively the $i$-th and $j$-th rows of $U$ and $V$. Suppose $u_i, v_j \in \reals^D$ for some finite $D$, we know that we can always choose $D \le \min\{m, n\}$. Let $w$ be a random vector drawn uniformly from the unit sphere in $\reals^D$. For any vector $v \in \reals^D$, let $\tilde{v} \defeq \sign(v^\top w) v$ and obtain the matrices $\tilde{U}, \tilde{V}$ from $U$ and $V$ by applying this transformation to all the rows of $U$ and $V$.

We define $\tilde{M} = \tilde{U} \tilde{V}^\top$. Note that as the only difference between $U$ and $\tilde{U}$ (resp. $V$ and $\tilde{V}$) is sign changes for some subset of the rows, we have $\norm{\tilde{U}}_{2, \infty} = \norm{U}_{2, \infty}$ (resp. $\norm{\tilde{V}}_{2, \infty} = \norm{V}_{2, \infty}$). Also for any $i, j$, $|u_i^\top v_j| = |\tilde{u}_i^\top \tilde{v}_j|$, thus the entries of $M$ and $\tilde{M}$ may only differ in sign. Thus, $\tilde{M} \in \mMmax$ irrespective of the random choice of the vector $w$. In order to complete the proof by the probabilistic method, it suffices to show that, 
\begin{align}
\E_{w} \left[ \E_{x \sim \nu} [\tilde{M_x}] \right] &\ge \frac{2}{\pi K} \ell(\nu, M, \bzero_\inds). \label{eqn:probmethod}
\end{align}

Let $\theta(u, v)$ denote the angle between any two vectors $u, v \in \reals^D$.
For a vector $w$ drawn at random from the unit sphere it can be verified that $ \Pr[\sign(\tilde{u} ^\top \tilde{v}) \neq \sign(u^\top v)] = \frac{\theta(u, v)}{\pi} $. We will show: 
\begin{align}
	\E_w [\tilde{u}^\top \tilde{v}] &\ge \frac{2(u^\top v)^2}{\pi \norm{u}\norm{v}}. \label{eqn:probmethod2}
\end{align}

It suffices to show the above for the case when $u^\top v \ge 0$ as flipping the sign of $v$ affects neither $\tilde{v}$ nor $(u^\top v)^2$. Note that in this case $\theta(u, v) \in [-\pi/2, \pi/2]$. To prove Eq.~\eqref{eqn:probmethod2}, we use:

\begin{lemma} \label{lem:coslemma} For any $\theta \in [-\pi/2, \pi/2]$, we have that, $1 - 2 \theta/\pi \ge 2 \cos(\theta)/\pi$. 
\end{lemma}
\begin{proof} [Proof of Lemma~\ref{lem:coslemma}]
Let $f : [-\pi/2, \pi/2] \rightarrow \reals$, where $f(\theta) = 1 - 2\theta/\pi - 2 \cos\theta/\pi$. Note that since $|\sin \theta| \le 1$ for all $\theta$, $f^\prime(\theta) \le 0$, so $f$ is decreasing and $f(\pi/2) = 0$.  
\end{proof}

Now, using the fact that $\cos(\theta(u, v)) = u^\top v/(\norm{u}\norm{v})$, we get, 
\begin{align*}
\E_w [\tilde{u}^\top \tilde{v}] &= u^\top v \cdot \left( 1 - 2 \Pr[\sign(\tilde{u}^\top \tilde{v}) \neq \sign(u^\top v)] \right) \\
&= u^\top v \cdot \left(1 - \frac{2\theta(u, v)}{\pi}\right) \ge u^\top v \cdot \frac{2 \cos(\theta(u, v))}{\pi} = \frac{2 (u^\top v)^2}{\pi \norm{u} \norm{v}}. 
\end{align*}

In the last step, we used Lemma~\ref{lem:coslemma}.  For any fixed $x = (i, j) \in \inds$, we have that,
\[ \E_{w} [\tilde{M}_x] = \E_w[\tilde{u_i}^\top \tilde{v_i}] \ge \frac{2 (u_i^\top v)^2}{\pi \norm{u_i} \norm{v_j}} \ge \frac{2 M_x^2}{\pi K}. \]
The last inequality follows from the fact that $\mnorm{M} = \norm{U}_{2, \infty} \cdot \norm{V}_{2, \infty} \le K$ and that $M_x = u_i^\top v_j$. Taking expectation with respect to $\nu$ and applying Fubini's theorem establishes~\eqref{eqn:probmethod}.
\end{proof}

\begin{proof} [Proof of Theorem~\ref{thm:main_efficient}]
	Apart from the use of Proposition~\ref{prop:key-GW}, the proof is
	essentially the same as that of Theorem~\ref{thm:main_inefficient}.
	We bound the probability of the three undesirable events by $\delta/3$ each.
	First, $\FC$ succeeds with probability at least $1 - \delta/3$. 

	Second, provided $N \ge \sampc(\eps^2/(8 \pi^2 K^2), \delta/3, \mMmax)$, for $\hat{M}$ output
	by $\FC$, using Defn.~\ref{defn:samplecomp} it must hold
	for each $M \in \mMmax$ that with probability at least $1 - \delta/3$,
	\begin{align}
		\left| \ell(\mu, M, \hat{M}) - \ell(S, M, \hat{M}) \right| &\le \frac{\eps^2}{8 \pi^2 K^2} \label{eqn:apply-mhat-efficient} 
	\end{align}
	In particular, when $\beta = \eps^2/(8 \pi^2 K^2)$, this means that (assuming $\FC$ has not failed),
	\begin{align*}
		\norm{(M^\star - \hat{M})/{2}}^2_S \le \frac{1}{4} \norm{M^\star
		- M}_\mu^2 + \frac{\epsilon^2}{32 \pi^2 K^2} \le \frac{\eps^2}{16 \pi^2 K^2} + \frac{\eps^2}{32 \pi^2 K^2} \le \beta.
	\end{align*}
	As a result, and using the convexity of $\mMmax$ (cf. Lemma~\ref{lem:convexclosed}), $(M^\star - \hat{M})/2
	\in \V(\bzero_\inds, \beta, S; \mMmax)$.

	For $N \ge \sampc(\eps^2/(8 \pi^2 K^2) = \gamma^2/2, \delta/3, \mMmax)$,
	Proposition~\ref{prop:efficient} guarantees with probability at least $1
	- \delta/3$, that the output of MP~\ref{opt:efficient} satisfies
	$\|C\|_1 \ge 1/\mu_{\max}$. 

	Finally, since $(M^\star - \hat{M})/2 \in \V(\bzero_\inds, \beta, S; \mMmax)$, we have:
	\begin{align*}
    \ell(\nu_C, M^\star, \hat{M}) &\le \sup_{M \in \V(\bzero_\inds, \beta, S; \mMmax)}\!\!\!\!\!\!\!\!\!\!\!\!\!\!\!\! \ell(\nu_C, M, \bzero_\inds) 
    \le \frac{\pi K}{2} \cdot \!\!\!\!\! \sup_{M \in \V(\bzero_\inds, \beta, S; \mMmax)} \!\!\!\!\!\!\!\!\!\!\!\!\!\!\!\!\E_{x \sim \nu_C} [M_x] 
 	\le \frac{\pi K}{2} \cdot \gamma.
 	\end{align*}
	In the penultimate step we used Proposition~\ref{prop:key-GW} and in the last step we used the fact that $C$ is feasible for MP~\ref{opt:efficient}. Substituting the value of $\gamma$ completes the proof. 
\end{proof}

\newcommand{\chow}{\mathrm{Chow}}
\def\A{{\mathcal A}}
\def\F{{\mathcal F}}
\def\H{{\mathcal H}}
\def\M{{\mathcal M}}
\def\R{{\mathcal R}}
\def\U{{\mathcal U}}
\def\V{{\mathcal V}}
\def\W{{\mathcal W}}
\def\X{{\mathcal X}}
\def\Y{{\mathcal Y}}
\def\Z{{\mathcal Z}}
\def\tr{{\text{Tr}}}
\setlength{\parindent}{0pt}

\section{Online Partial Matrix Completion}
\label{sec:online}

This section uses tools from online convex optimization, in particular its application to games, to design an efficient gradient-based online algorithm with provable regret guarantee to find a near optimal confidence matrix to the partial matrix completion problem. 

\paragraph{Organization.} The organization of this section is as follows: we begin with a short background description on online convex optimization, its application to matrix completion problems and games in Section~\ref{sec:online-prelim}, explain the main motivations and intuitions of our online algorithm in Section~\ref{sec:online-motivation}, introduce the general protocol of online partial matrix completion in Section~\ref{sec:online-protocol} and the design of objective functions in Section~\ref{sec:online-objective}. We formally give the algorithm specification and the regret guarantee in Section~\ref{sec:online-alg}. Finally, we show the implication on statistical guarantee in Section~\ref{sec:online-to-offline}.

\subsection{Preliminaries}
\label{sec:online-prelim}

\paragraph{Online convex optimization.}

In online convex optimization, a player iteratively chooses a point $x_t \in \K \subseteq \reals^d$ at time step $t$,  and receives a convex loss function $h_t$, to which the player incurs loss $h_t(x_t)$. The performance of the player is measured by \emph{regret}, the total excess loss incurred by the player's decisions than the best single decision $x^*\in\K$. Formally, regret is given by the following definition
\begin{align*}
\regret_T\defeq \sum_{t=1}^T h_t(x_t)-\min_{x\in\K} \sum_{t=1}^T h_t(x). 
\end{align*}
The goal is to design algorithms that achieve sublinear regret, which means that with time the algorithm's decisions converge in performance to the best single decision in hindsight. For a survey of methods and techniques, see \citep{hazan2016introduction}.   

\paragraph{Online matrix prediction.} 
Online convex optimization has been proved useful in solving matrix prediction problems. In particular, \citet{hazan2012near}  give an efficient first-order online algorithm that iteratively produces matrices $M_t$ of low complexity that for any sequence of adversarially chosen convex, Lipschitz loss functions and indices $(i_t,j_t)$'s, the online matrix prediction algorithm gives a regret bound of $\tilde{\mathcal{O}}(\sqrt{(m+n)T})$. The linear dependence on the matrix dimension $m,n$ translates to the convergence in performance to the best complexity-constrained matrix in hindsight after seeing only square root of the total number of entries.

\paragraph{Games and regret.}
One important branch of theory developed in online convex optimization is its connection to finding the equilibrium point in two-player games, first discovered by \citet{freund1999adaptive}. In this framework, two players iteratively pick decisions $x_t\in\X, y_t\in\Y$ at time step $t$, after which a (possibly adversarially) chosen loss function $h_t(x,y)$ is revealed, where $h_t$ is convex in $x$ and concave in $y$. Player 1 incurs loss $h_t(x_t,y_t)$ and player 2 gains reward $h_t(x_t,y_t)$. The objective for the players is to produce a sequence of decisions $\{x_t\}_{t=1}^T,\{y_t\}_{t=1}^T$ that converges in performance to the best single $x^*,y^*\in\K$ in hindsight.  The regret for the two players are given by
\begin{align*}
\regret_T(\texttt{player1})&=\sum_{t=1}^T h_t(x_t,y_t)-\min_{x\in\X}\sum_{t=1}^T h_t(x,y_t), \\ \regret_T(\texttt{player2})&=\max_{y\in\Y}\sum_{t=1}^T h_t(x_t,y)-\sum_{t=1}^T h_t(x_t,y_t).
\end{align*}
Under distributional assumption of the function $h_t$'s, namely if $h_t$'s are bilinear, i.i.d. stochastically chosen according to some distribution, then the following holds:
\begin{align*}
& \ \ \ \ \ \left|\E\left[\sum_{t=1}^T h_t(x_t,y_t)\right]-\min_{x\in\X}\max_{y\in\Y} \E\left[\sum_{t=1}^T h_t(x,y)\right]\right|\\
& \le \max\{\regret_T(\texttt{player1}),\regret_T(\texttt{player2})\}. 
\end{align*}
For such reasons, we define a notion of game regret, denoted $\gregret_T$, to be the maximum of the regret incurred by the two players. Sublinear game regret can be used to compute the game equilibrium \citep{freund1999adaptive}.  

\subsection{Motivations}
\label{sec:online-motivation}
Two important aspects lie at heart of the general motivation for considering partial matrix completion in an online setting. First, in many applications, the observation pattern is more general than a fixed distribution. It can be a changing distribution or be comprised of adversarial observations. Second, our online algorithm incrementally updates the solution via iterative gradient methods, which is more efficient than the offline methods. 
\subsubsection{Game theoretic nature of partial matrix completion}
With the preliminary introduction on saddle-point problems, we will describe their connection to our problem of interest - finding the optimal confidence matrix $C$ in partial matrix completion problems. Recall that we want to find $C$ that (1) has maximal coverage under the constraint that (2) the deviation between any two possible completions $M^1$ and $M^2$ with respect to $C$ is small. 

With this objective in mind, we can think of the following two-player game. Player 1 plays a confidence matrix $C$, player 2 plays a pair of possible matrix completions $(M^1,M^2)$. The goal of player 1 is to (a) maximize the coverage of $C$ and (b) minimize the deviation between $M^1$ and $M^2$ with respect to $C$. The goal of player 2 is to maximize the deviation between $M^1$ and $M^2$ with respect to $C$. The equilibrium point in this problem is exactly given by a confidence matrix that simultaneously has high coverage and with respect to which the deviation between any two possible completions $M^1$ and $M^2$ is small. The game theoretic nature of the partial matrix completion problem leads us to consider designing a provably low-game-regret online two-player algorithm. 

\subsubsection{Soft constraints}
It is worth noting that this formulation differs from the MP\ref{opt:inefficient} and MP\ref{opt:efficient} in the absence of hard constraints. In particular, instead of imposing an $\eps$-margin on the deviation between any two possible completions with respect to the confidence matrix, we formulate this objective in the objective functions for player 1 and 2. This allows us to perform fast gradient-based algorithm, which will be detailed soon in the Algorithm~\ref{alg:odd}. The set of possible matrix completions is formally given by the version space, i.e. the set of matrices of low complexity, with bounded norms, and deviate little from the observed data. To avoid computationally expensive projections, Algorithm~\ref{alg:odd} further removed the constraint on deviation from observations and replaced with a soft constraint that penalizes the completion's deviation from observations. Similar techniques have been seen in \citep{mahdavi2012trading}. 

\subsection{Online PMC General Protocol}
\label{sec:online-protocol}
With the motivations explained, we are ready to introduce the general protocol of online matrix completion and then give the details of our objective functions. In the online partial matrix completion problem, the algorithm acts for two players and follows the following protocol. At time step $t$, 
\begin{enumerate}
\item Player 1 picks a confidence matrix from a constraint set $C_t\in\C$, where $\C$ is defined below. Player 2 picks two matrices $M_t^1,M_t^2\in\mMmax$.
\item An adversary reveals a tuple $(x_t,o_t)\in\X\times [-1,1]$. Based on this tuple, a function is constructed, denoted $h_t(C,M^1,M^2)$, which  is concave in $C$ and convex in $M^1,M^2$. 
\item Player 1 receives reward $h_t(C_t,M_t^1,M_t^2)$. Player 2 incurs loss $h_t(C_t,M_t^1,M_t^2)$. 
\end{enumerate}

\ignore{
Similar to the offline setting, we have two objectives for our desired confidence matrix. The first is for the confidence matrix to exhibit confidence over a large subset of the entries, measured by the effective support size $H(\cdot)$. The second objective is to exhibit low confidence at the entries where two differing yet ``reasonable'' completions show large deviations. Here, ``reasonable'' refers to those completions that (approximately) agree with previous observations and have low complexity, analogous to being in the version space. Letting $M^1,M^2$ be two  ``reasonable'' completions, the goal is to find $C$ such that 
$\max_{M^1,M^2} \ \sum_{x\in\X} C_{x}(M_x^1-M_x^2)^2$
is small. 
}

The convex-concave function $h_t$,  detailed in Section \ref{sec:online-objective},   measures the coverage of $C$, the deviation between $M^1$ and $M^2$ with respect to $C$, and the deviation of $M^1,M^2$ from the observations. 


Note that in previous sections, we took $\C=[0,1]^{\X}$. 
In this section we consider a more general case, and consider two choices for $\C$. The first choice is the unit simplex $\Delta_{\X}\subset\reals^{\X}$. This is natural as it induces a probability distribution of completion confidence over the entries. The second choice is the scaled hypercube $\C=\left[0,\frac{1}{(m+n)^{3/2}}\right]^{\X}$. The scaling is for the mere purpose for analysis and representation of theorems. In \ref{sec:online-to-offline}, we will show a reduction from the online result to the offline guarantee. 

\subsection{Designing Convex-Concave Objectives}
\label{sec:online-objective}
The previous sections motivate the design of the convex-concave function $h_t$, where $\alpha, \theta>0$ are positive parameters: 
\begin{align*}
h_t(C,M^1,M^2) &= H(C)-  \alpha G(C,M^1,M^2)+  \alpha\theta f_t(M^1,M^2),
\end{align*}
where the three components of $h_t$ are:
\begin{enumerate}

\item $H(\cdot)$ is a measure of the effective support size of $C$, which can be taken  to be either: 
\begin{enumerate}
\item \label{assumptionH1} entropy, i.e. $H(C)=-\sum_{x\in\X}C_x\log(C_x)$ defined over the simplex $\C=\Delta_{\X}$, or
\item \label{assumptionH2} $H(C)=\|C\|_1$, defined over the scaled hypercube $\C=\left[0,\frac{1}{(m+n)^{3/2}}\right]^{\X}$. 
\end{enumerate}

\item $G(C,M^1,M^2)\defeq \sum_{x\in\X} C_x(M_x^1-M_x^2)$, a linear relaxation of $\sum_{x\in\X} C_x(M_x^1-M_x^2)^2$ (see Theorem \ref{thm:linrelax} and Corollary \ref{cor:linrelax} in Appendix \ref{app:online}), which measures the deviation of two completions $M^1,M^2$ with respect to $C$.
\item $f_t(M^1,M^2)\defeq\left[(M^1_{x_t}-o_t)^2+(M^2_{x_t}-o_t)^2\right]$, which measures the deviation of the two completions from observation made at time $t$. This serves as a soft constraint where $M^1,M^2$ minimizing $h_t$ will have values close to $o_t$ at entry $x_t$. 
\end{enumerate}
With slight abuse of notation, we denote $M_t\defeq (M_t^1,M_t^2)$ in the following sections for convenience and presentation clarity.  Note that $h_t$ is concave in $C$ and convex in $M^1,M^2$. Similar to the previous section, remark that we can, albeit suffering a constant in all performance metrics, consider that all observations $o_t$ are zero and maintain a single matrix in place of $M_t^1,M_t^2$ that measures the radius of the version space. 

\subsection{Online Dual Descent (\texttt{ODD}): online gradient-based algorithm for partial matrix completion}
\label{sec:online-alg}
Formally, we propose Online Dual Descent (\texttt{ODD}, Algorithm \ref{alg:odd}):
\begin{algorithm}[h!]
\caption{Online Dual Descent for Partial Matrix Completion (\texttt{ODD})}
\label{alg:odd}
\begin{algorithmic}[1]
\STATE \textbf{Input}: Gradient-based online coverage and matrix update functions $\A_C, \A_M$, \ignore{Online gradient-based MC algorithm $\A$ (see \ref{sec:BlackBoxProof} for an example),}  parameters $\eta>0$, $\alpha, \theta>0$.
\STATE Initialize $C_1, M_1\leftarrow \A_C(\emptyset), \A_M(\emptyset)$.
\FOR {$t=1,2,...,T$}
\STATE Player 1 plays $C_t$; player 2 plays $M_t$.
\STATE Adversary draws tuple $(x_t,o_t)$ and constructs function $h_t$ with parameters $\alpha,\theta$. 
\STATE Player 1 receives reward $h_t(C_t,M_t)$, player 2 incurs loss $h_t(C_t,M_t)$. 
\STATE Player 1 updates $C_{t+1} \leftarrow\A_C(C_t,M_t)$. Player 2 updates $M_{t+1}\leftarrow \A_M(C_{t},M_{t},(x_t,o_t))$.
\ignore{
\STATE  Take the regularization function $R$ to be  the negative entropy $-H$. Update {
$$\nabla R(\hat{C}_{t+1})\leftarrow\nabla R(C_t)+\eta\nabla r_t(C_t).$$ 
}
Project w.r.t. Bregman divergence of the negative entropy function,
$$C_{t+1}=\underset{C\in\Delta_{\X}^{\beta}}{\argmin} \  B_R(C,\hat{C}_{t+1}).$$
\eh{suggestion: write this as algorithm, similar to the update for $M$, and assume low regret over the simplex. no need to rederive the regret bounds over the simplex}
}
\ENDFOR
\STATE \textbf{output}:s $\bar{C}=\frac{1}{T}\sum_{t=1}^T C_t$. 
\end{algorithmic}
\end{algorithm}

Here, 
$\A_C: \C \times (\mMmax)^2 \rightarrow \C$ and 
$\A_M: \C \times (\mMmax)^2 \times (\X, [-1, 1]) \rightarrow  (\mMmax)^2.$
The detailed analysis of the algorithm will be deferred to the following section and Appendix \ref{app:online}, but we will first state its regret guarantee:




\begin{corollary}\label{sec:MainResult}
For any sequence of $\{(x_t,o_t)\}_{t=1}^T$, Algorithm \ref{alg:odd} gives the following  regret guarantee on the obtained set
$\{(C_t,M_t)\}_{t=1}^T$, such that for settings \ref{assumptionH1}, \ref{assumptionH2},
\begin{align*}
\gregret_T&\defeq \max\left\{\regret_T^{\mA_C}, \alpha\cdot\regret_T^{\mA_M}\right\} \le \tilde{O}(K \alpha \theta \sqrt{(m+n)T}).
\end{align*} 
\end{corollary}
We henceforth denote the upper bound on the regret by 
$$  \gregret_T \le B_T . $$

The main theorem regarding the game regret of Algorithm \ref{alg:odd} is implied by the existence of low regret guarantee for gradient-based subroutines $\A_C, \A_M$, outlined in the following theorems. The details of the subroutines and analysis are deferred to Appendix~\ref{app:online}. 

\begin{theorem}
\label{thm:c-subroutine-regret}
Denote at every time step $t$, consider the concave reward function $r_t(C)\defeq H(C)-  \alpha G(C,M_t)$.
There exists a sub-routine gradient-based update $\A_C$ (see Alg.~\ref{appalg:omd-c} for an example) with the following regret guarantee w.r.t. $r_t$:
\ignore{Algorithm $\ref{alg}$ guarantees the following regret bound: with $\mathcal{O}$ hiding all constants independent of $m,n,T$:}
\begin{enumerate}
\item For setting \ref{assumptionH1}, 
\begin{align*}
\regret_T^{\A_C}\defeq\max_{C\in\C}\sum_{t=1}^T r_t (C) -\sum_{t=1}^{T} r_t (C_t)\le O(\alpha\sqrt{\log(mn)T})=\tilde{O}(\alpha\sqrt{ T}).
\end{align*}
\item For setting \ref{assumptionH2},  
\begin{align*}
\regret_T^{\A_C}\defeq\max_{C\in\C}\sum_{t=1}^T r_t (C) -\sum_{t=1}^{T} r_t (C_t)\le O(\alpha\sqrt{(m+n)T}).
\end{align*}
\end{enumerate}
\end{theorem}
\ignore{
\begin{proof}
See \ref{app:c-subroutine-proof}.
\end{proof}
}

\begin{theorem} [adapted from \cite{hazan2012near}]
\label{thm:M-subroutine-regret}
Denote at every time step $t$, consider the concave reward function $\gamma_t(M)\defeq G(C_t,M)-\theta f_t(M)$. There exists a sub-routine gradient-based update $\A_M$ (see Alg.~\ref{appalg:omd-m} for an example) such that, under either setting \ref{assumptionH1} or \ref{assumptionH2}, the following regret guarantee w.r.t. $\gamma_t$ holds:
\begin{align*}
\regret_T^{\A_M}\defeq\max_{M\in\mMmax\times\mMmax}\sum_{t=1}^T \gamma_t(M)-\sum_{t=1}^T \gamma_t (M_t)\le O(K \theta \sqrt{(m+n)T}).
\end{align*}
\end{theorem}

\subsection{Offline implications}
\label{sec:online-to-offline}
In this section we show that how the regret guarantee we obtained in the online setting translates to an offline performance guarantee on $C$. The offline implications hold under the following assumptions of the revealed entries:
\begin{assumption}
At each time step $t$, the index $x_t=(i_t, j_t)$ is sampled according to some unknown sampling distribution $\mu$, and the observation $o_t=M^{\star}(i_t,j_t)$, where $M^\star\in\mMmax$ is the ground truth matrix.
\end{assumption}
Consider the following empirical and general version spaces:
\begin{align*}
\V_{T,\delta}&\defeq \left\{M\in\mMmax\mid \frac{1}{T}\sum_{t=1}^T (M_{x_t}-o_t)^2\le\delta\right\}, \\ \V_{\delta}&\defeq \left\{M\in\mMmax\mid \underset{(x,o)\sim\D}{\E}[(M_{x}-o)^2]\le\delta\right\}. 
\end{align*}

\begin{lemma}
\label{thm:online-to-offline} 
After $T$ iterations, and assume that for some $\delta>0$,
$$ \frac{1}{\theta}\left(2D+ \frac{\regret_T^{\mA_M}}{T}\right) \le \frac{\delta^{2/3}}{2},$$ 
with $D=1$ in setting \ref{assumptionH1}, and $D=\sqrt{m+n}$ in setting \ref{assumptionH2}. 
The following properties hold on the obtained $\bar{C}\defeq \frac{1}{T}\sum_{t=1}^T C_t$ returned by Algorithm \ref{alg:odd}: with probability $\ge 1-\exp\left(-\frac{\delta^{4/3}T}{512}\right)$,
\begin{align*}
H(\bar{C}) -   \max_{M\in \V_{T,\delta}^2} \alpha \cdot G(\bar{C},M)\ge \max_{C\in\C}\min_{M\in\V_{\delta^{2/3}}^2} \left\{ H(C)-  \alpha \cdot G(C,M)\right\}-2\alpha\theta\delta- \frac{B_T} {T}.
\end{align*}
\end{lemma}

Lemma~\ref{thm:online-to-offline} implies the following guarantee.

\begin{theorem}
\label{cor:online-to-offline}
Suppose the underlying sampling distribution is $\mu$. Let $C_{\mu} \in\C$ be its corresponding confidence matrix. In particular, for setting \ref{assumptionH1}, $C_{\mu}=\mu$, i.e.  $(C_{\mu})_{ij}=\mathbb{P}_{\mu}((i,j)\text{ is sampled})$; for setting \ref{assumptionH2}, $C_{\mu}$ satisfies that $C_{\mu}/\|C_{\mu}\|_1=\mu$. Then, for any $\delta>0$, Algorithm~\ref{alg:odd} run with $\alpha=O(\delta^{-1/6})$ returns a $\bar{C}$ that guarantees the following bounds: with probability at least $1-c_1\exp(-c_2\delta^2T)$ for some universal constants $c_1,c_2>0$,
\begin{enumerate}
\item For setting \ref{assumptionH1}, take $\theta=O(\delta^{-2/3})$, after $T=\tilde{O}(\delta^{-2}K^2(m+n))$ iterations,
\begin{enumerate}
\item $H(\bar{C})\ge H(C_{\mu})-O(\delta^{1/6})$.
\item $\underset{M^1,M^2\in\V_{\delta}}{\max}\ G(\bar{C},M^1,M^2)\le O(\delta^{1/6}\log(mn))$.
\end{enumerate}
\item For setting \ref{assumptionH2}, take $\theta=O(\delta^{-2/3}\sqrt{m+n})$, after $T=\tilde{O}(\delta^{-2}K^2(m+n))$ iterations, 
\begin{enumerate}
\item $\|\bar{C}\|_1\ge \|C_{\mu}\|_1-O(\delta^{1/6}\sqrt{m+n})$.
\item $\underset{M^1,M^2\in\V_{\delta}}{\max} G(\bar{C},M^1,M^2)\le O(\delta^{1/4}\sqrt{m+n})$.
\end{enumerate}
\end{enumerate}
\end{theorem}

\begin{remark}
We explain the implication of the above theorem. Suppose the sampling distribution is supported uniformly across a constant fraction $0<c\le 1$ of the entire matrix. This implies that (1) the obtained confidence matrix $\bar{C}$ has a coverage lower bounded by $1-\delta^{1/6}$ fraction of the true distribution, and (2) $\bar{C}$ induces a weighted maximal distance on the version space $\V_{\delta}$: 
\begin{align*}
\max_{M^1, M^2\in\V_{\delta}} \left\{\frac{1}{\|\bar{C}\|_1}\sum_{i\in[m], j\in[n]}\bar{C}_{ij}(M_{ij}^1-M_{ij}^2)^2\right\}\le O(\delta^{1/6}).
\end{align*}
Also see Corollary~\ref{cor:online} for this example.
\end{remark}

\section{Supporting Proofs for Appendix \ref{sec:online}}
\label{app:online}

\subsection{Linear Relaxation}
This section follows similarly to that of Proposition~\ref{prop:key-GW}.
\ignore{
First, we note a key observation of the constraint set $\mMmax$:
\begin{lemma}
\label{lem:max-set-closed}
$\mM_K^{\max}$ is closed under the operation of negating any subset of. 
\end{lemma}
\begin{proof}
Let $M\in\mM_K^{\max}$. It suffices to show for fixed $i\in[m]$, $\|M'\|_{\max}\le k$, where $M'_i=-M_i$ and $M'
_{i'}=M_{i'}$ if $i'\neq i$. By assumption, $\exists U,V$ such that $UV^T=M$ and $\|U\|_{2,\infty}\|V\|_{2,\infty}\le k$. Let $U'$ be such that $U'_i=-U_i$ and $U'
_{i'}=U_{i'}$ if $i'\neq i$. $U'V^T=M'$ and $\|U'\|_{2,\infty}=\|U\|_{2,\infty}$, which implies that $\|M'\|_{\max}\le k$. The general case follows by induction. 
\end{proof}

With Lemma \ref{lem:max-set-closed}, we can establish the following inequality for version spaces $\V_0$ centered around the zero matrix:
}
\begin{theorem}
\label{thm:linrelax}
Consider the following two functions defined on $\C$:
\begin{align*}
\tilde{g}(C)\defeq \sup_{M\in\V_0} \sum_{x\in\X} C_xM_x^2, \ \ g(C)\defeq \sup_{M\in\V_0}\sum_{x\in\X} C_xM_x.
\end{align*}
Assume that $\V_0= \mathcal{M}_{K}^{\max}\cap S$, where $S$ is a constraint set that contains $0$ that is closed under the operation of negating any subset of entries. Then $g(\cdot)$ is non-negative and $\forall C\in\C$,  there holds
\begin{align*}
\tilde{g}(C)\le \frac{\pi K}{2} g(C).
\end{align*}
\end{theorem}
\begin{proof}
That $g(\cdot)$ is non-negative follows from the assumption that $0\in\V_0$. It suffices to show the inequality for $\C=\Delta_{\X}$, since $\forall C\in\left[0,\frac{1}{(m+n)^{3/2}}\right]^{\X}$, we can consider $C'=\frac{C}{\|C\|_1}\in\Delta_{\X}$, and $\tilde{g}(C')\le \frac{\pi K}{2}g(C')$ implies $\tilde{g}(C)\le \frac{\pi K}{2} g(C)$.

By compactness of $\V_0$, $\exists M\in\V_0$ such that the value of $\tilde{g}(C)$ is achieved. Moreover, since $\C=\Delta_{\X}$, $C$ defines a probability distribution on $\X$. 
Then, it suffices to show that $\exists \tilde{M}\in\V_0$ such that 
\begin{align*}
\sum_{x\in\X}C_xM_x^2=\E_C[M_x^2]\le \frac{\pi K}{2}\E_C[\tilde{M}_x]\le \frac{\pi K}{2} \sum_{x\in\X} C_x\tilde{M}_x.
\end{align*}
We start the construction of $\tilde{M}$. By assumption that $\V\subseteq\mMmax$, $\exists U\in\mathbb{R}^{m\times d},V\in\mathbb{R}^{n\times d}$, $d=\text{rank}(M)$ such that $M=UV^T$ and $\|U\|_{2,\infty}\|V\|_{2,\infty}\le K$. Denote as $u_i, v_j$ the $i$-th and $j$-th row of $U,V$, respectively. Let $S^d$ denote the unit sphere in $\mathbb{R}^d$. Draw a random vector $w\sim S^d$ uniformly at random and consider its inner product with each of the $u_i, v_j$'s. Define $\tilde{U}\in\mathbb{R}^{m\times d}, \tilde{V}\in\mathbb{R}^{n\times d}$ in the following way: with $\tilde{u}_i, \tilde{v}_j$ being the $i$-th and $j$-th row of $\tilde{U}, \tilde{V}$,
$$\tilde{u}_i\defeq \text{sign}(w^Tu_i)u_i, \ \ \tilde{v}_j\defeq \text{sign}(w^Tv_j)v_j.$$
Note that $\|\tilde{U}\|_{2,\infty}=\|U\|_{2,\infty}$, $\|\tilde{V}\|_{2,\infty}=\|V\|_{2,\infty}$. Therefore, together with the assumption that $S$ is closed under negation over any subset of entries, $\tilde{M}\defeq\tilde{U}\tilde{V}^T\in\V_0$. 

Consider the hyperplane parametrized by $w$, $P_w\defeq\{x\in\mathbb{R}^d\mid w^Tx=0\}$, then
\begin{align*}
\mathbb{P}(\text{sign}(u_i^Tv_j)\neq \text{sign}(\tilde{u}_i^T\tilde{v}_j))&=\mathbb{P}(\text{sign}(w^Tu_i)\neq \text{sign}(w^Tv_j))\\
&=\mathbb{P}(u_i,v_j\text{ are separated by }P_w)=\frac{\arccos\left(\frac{u_i^Tv_j}{\|u_i\|_2\|v_j\|_2}\right)}{\pi}. 
\end{align*}
Taking expectation over the distribution of the random vector $w$,
\begin{align*}
\E_w[\tilde{u_i}^T\tilde{v_j}]=u_i^Tv_j\left(1-\frac{2\arccos\left(\frac{u_i^Tv_j}{\|u_i\|_2\|v_j\|_2}\right)}{\pi}\right)\ge \frac{2(u_i^Tv_j)^2}{\pi\|u_i\|_2\|v_j\|_2}\ge  \frac{2(u_i^Tv_j)^2}{\pi k} 
\Longleftrightarrow \E_w[\tilde{M}_{ij}]\ge \frac{2M_{ij}^2}{\pi k}.
\end{align*}
Taking expectation over distribution $C$ and applying Fubini's Theorem, 
\begin{align*}
\E_C[M_x^2]\le \frac{\pi K}{2} \E_w \E_C[\tilde{M}_x],
\end{align*}
which implies that there exists an instance of $\tilde{M}\in\V_0$ such that
\begin{align*}
\sum_{x\in\X} C_xM_x^2\le \frac{\pi K}{2}\sum_{x\in\X} C_x\tilde{M}_x.
\end{align*}
\end{proof}

\begin{corollary}
\label{cor:linrelax}
The following inequality holds:
\begin{align*}
\sup_{M^1,M^2\in \V} \sum_{x\in\X} C_x (M^1_x-M^2_x)^2\le \pi K
\sup_{M^1,M^2\in \V} \sum_{x\in\X} C_x (M^1_x-M^2_x),
\end{align*}
for max-norm constrained, symmetric version space $\V$ around a given matrix. 
\end{corollary}
\begin{proof}
By transformation to a version space $\V_0$ centered around the zero matrix, if $M^1,M^2\in \V$, then $M\defeq \frac{M^1-M^2}{2}\in\V_0$, where $\V_0=\mMmax\cap S$ and $S$ is a constraint set that is closed under negation under any subset of entries. The result is subsumed by Theorem $\ref{thm:linrelax}$.
\end{proof}

\subsection{Proof of regret guarantees}
\subsubsection{Proof of Theorem \ref{thm:c-subroutine-regret}}
\label{app:c-subroutine-proof}

\begin{reptheorem}{thm:c-subroutine-regret}
Denote at every time step $t$, consider the concave reward function $r_t(C)\defeq H(C)-  \alpha G(C,M_t)$.
There exists sub-routine gradient-based update $\A_C$ (see Alg.~\ref{appalg:omd-c} for an example) with the following regret guarantee w.r.t. $r_t$:
\ignore{Algorithm $\ref{alg}$ guarantees the following regret bound: with $\mathcal{O}$ hiding all constants independent of $m,n,T$:}
\begin{enumerate}
\item For setting \ref{assumptionH1}, 
\begin{align*}
\regret_T^{\A_C}\defeq\max_{C\in\C}\sum_{t=1}^T r_t (C) -\sum_{t=1}^{T} r_t (C_t)\le O(\alpha\sqrt{\log(mn)T})=\tilde{O}(\alpha\sqrt{ T}).
\end{align*}
\item For setting \ref{assumptionH2},  
\begin{align*}
\regret_T^{\A_C}\defeq\max_{C\in\C}\sum_{t=1}^T r_t (C) -\sum_{t=1}^{T} r_t (C_t)\le O(\alpha\sqrt{(m+n)T}).
\end{align*}
\end{enumerate}
\end{reptheorem}

\begin{proof}[Proof of Theorem~\ref{thm:c-subroutine-regret}]
We divide the proof into two parts, corresponding to two different choices of $H(\cdot)$ and the corresponding $\C$. Both use online mirror descent (OMD) step as updates. The analysis for the entropy case is slightly more involved due to the gradient behavior at the boundary. We will begin with outlining the algorithm:

\begin{definition} [Bregman divergence]
Let $R:\Omega\rightarrow \mathbb{R}$ be a continuously-differentiable, strictly convex function defined on a convex set $\Omega$. The \emph{Bregman divergence} associated with $R$ for $p, q\in\Omega$ is defined by
\begin{align*}
B_R(p,q)=R(p)-R(q)-\langle \nabla R(q), p-q\rangle.
\end{align*}
In particular, Bregman divergence measures the difference between $R(p)$ and the first-order Taylor expansion of $R(p)$ around $q$. 
\end{definition}

\begin{algorithm}[h!]
\caption{$\A_C$}
\label{appalg:omd-c}
\begin{algorithmic}[1]
\STATE Input: previous $C_t$, completions $M_t=(M_t^1,M_t^2)$. 
\STATE Require: step-size $\eta$, regularization function $R$.
\IF{input is empty}
\STATE Set $(\hat{C}_{t+1})_x=e^{-1}$, $\forall x\in\X$ in setting \ref{assumptionH1}; set $\hat{C}_{t+1}=\mathbf{0}_{\X}$ in setting \ref{assumptionH2}.
\ELSE
\STATE Update $\hat{C}_{t+1}$ via $\nabla R(\mathbf{vec}(\hat{C}_{t+1}))=\nabla R(\mathbf{vec}(C_t))+\eta\nabla r_t(\mathbf{vec}(C_t))$. 
\ENDIF
\STATE Obtain $C_{t+1}$ via Bregman projection: $C_{t+1}=\underset{C\in\C'}{\argmin} \ B_{R}(\mathbf{vec}(C),\mathbf{vec}(\hat{C}_{t+1}))$.
\STATE Output $C_{t+1}$.
\end{algorithmic}
\end{algorithm}
$\C'$ is taken to be
\begin{enumerate}
\item For setting~\ref{assumptionH1}, $\C'=\Delta_{\X}^{\beta}\defeq \{C\in\Delta_{\X}:\min_{x\in\X} \ C_x\ge e^{1-\beta}\}$. 
\item For setting~\ref{assumptionH2}, $\C'=\C=\left[0,\frac{1}{(m+n)^{3/2}}\right]^{\X}$, and 
\end{enumerate}
The square root diameter $D_R$ of $R(\cdot)$ over $\mathbf{vec}(\C')$ is given by 
\begin{align*}
D_R\defeq \sqrt{\underset{{X,Y\in\C'}}{\max}\{R(\mathbf{vec}(X))-R(\mathbf{vec}(Y))\}}. 
\end{align*}

\paragraph{$\ell_1$-norm and cube.}
In setting~\ref{assumptionH2}, take $R:\mathbf{vec}(\C')\subset \mathbb{R}^{mn}\rightarrow \mathbb{R}$ given by $R(x)=\frac{1}{2}\|x\|_2^2$. Then $D_R\le \frac{1}{2\sqrt{m+n}}$. The regret guarantee for online mirror descent also depends on the bound on local norms of the gradients. In particular, the local norm at time $t$ is a function mapping from $\mathbb{R}^{mn}$ to $\mathbb{R}_{++}$ given by
\begin{align*}
\|x\|_t^*=\sqrt{x^{\top}\nabla^2R(\tilde{C})x},
\end{align*}
where $\tilde{C}$ is some convex combination of $C_t$ and $C_{t+1}$ satisfying 
\begin{align*}
R(\mathbf{vec}(C_t))&=R(\mathbf{vec}(C_{t+1}))+\nabla R(\mathbf{vec}(C_{t+1}))^{\top}\mathbf{vec}(C_t-C_{t+1})\\
& \ \ \ \ \ +\frac{1}{2}\mathbf{vec}(C_t-C_{t+1})^{\top} \nabla^2 R(\mathbf{vec}(\tilde{C})) \mathbf{vec}(C_t-C_{t+1}). 
\end{align*}
Note that $\forall x\in\mathbf{vec}(\C')$, $\nabla^2 R(x)=I_{mn}$, and thus 
\begin{align*}
G_R^2\defeq {\|\mathbf{vec}(\nabla r_t(C_t))\|_t^*}^2=\|\mathbf{vec}(\nabla r_t(C_t))\|_2^2\le (1+2\alpha)^2 mn, 
\end{align*}
where inequality follows from that $(\nabla r_t(C))_{x}=1-\alpha(M_x^1-M_x^2)\le 1+2\alpha$. 

By standard Online Mirror Descent (OMD) analysis using the diameter and gradient bounds, by taking $\eta=\frac{D_R}{G_R\sqrt{T}}$, we have the regret bound of $r_t$'s over the cube:
\begin{align*}
\max_{C\in\C'} \sum_{t=1}^T r_t(C)-\sum_{t=1}^T r_t(C_t)\le D_RG_R\sqrt{T}= O(\alpha \sqrt{ (m+n)T}). 
\end{align*}

\paragraph{Entropy and simplex.}
In this setting, take $R(\mathbf{vec}(X))=-H(X)$, the negative entropy function, $D_R^2=\log(mn)$. The proof will be divided into two parts: (1) we first show low regret of $\A_C$ w.r.t. $\Delta_{\X}^{\beta}$, then (2) we show that the best $C$ in $\Delta_{\X}^{\beta}$ exhibits approximately the same performance as the best $C$ in $\Delta_{\X}$. 

Let $\gamma\defeq \beta+2\alpha$. For any $C\in\Delta_{\X}^\beta$, the gradient of the reward function is bounded by: 
\begin{align*}
\|\mathbf{vec}(\nabla r_t(C))\|_\infty&\le \|\nabla H(C)\|_\infty+\alpha\|\nabla_C G(C,M_t)\|_\infty & \text{$\Delta$-inequality}\\
&= \max_{x\in \X}|-1-\log C_x| + \alpha\cdot\max_{x\in\X} |(M_t^1-M_t^2)_x| & \text{definition of $\|\cdot\|_{\infty}$, $H$, $G$}\\
&\le 1+ \log(e^{\beta-1})+2\alpha & \text{$C\in\Delta_{\X}^\beta$, $M_t^1,M_t^2\in\mMmax$}\\
&\le \gamma.
\end{align*}
Thus, for some convex combination $\tilde{C}$ of $C_t$ and $C_{t+1}$,
\begin{align*}
G_R^2={\|\mathbf{vec}(\nabla r_t(C_t))\|_t^*}^2=\sum_{x\in\X} \tilde{C}_x(\nabla r_t(C_t))_x^2\le \|\mathbf{vec}(\nabla r_t(C_t))\|_\infty^2\|\tilde{C}\|_1\le \gamma^2,
\end{align*}
where the last last inequality follows from $\|\mathbf{vec}(\nabla r_t(C))\|_{\infty}\le \gamma$ and $\|\tilde{C}\|_1=1$.

By standard Online Mirror Descent (OMD) analysis using the diameter and gradient bounds, we have the regret bound of $r_t$'s over the constrained unit simplex $\Delta_{\X}^{\beta}$ as the following:
\begin{align*}
\max_{C\in\Delta_\X^\beta}\sum_{t=1}^T r_t(C)-\sum_{t=1}^Tr_t(C_t)\le D_RG_R\sqrt{T}=O(\alpha\sqrt{ \log(mn)T})=\tilde{O}(\alpha\sqrt{T}),
\end{align*}
when taking $\eta=\frac{D_R}{G_R\sqrt{T}}$.

With the regret bound established w.r.t. $\Delta_{\X}^{\beta}$, it is left to show the following inequality, which justifies constraining the feasible set to $\Delta_{\X}^{\beta}$: 

\ignore{
\begin{lemma} 
\label{lem:constrained-simplex}
The following inequality holds: 
\begin{align*}
\max_{C\in\Delta_{\X}}\sum_{t=1}^T r_t(C) -\max_{C\in \Delta_{\X}^\beta} \sum_{t=1}^T r_t(C) \leq 2 e^{1-\beta} T \log(mn) . 
\end{align*}
\end{lemma}
\begin{proof}[Proof of Lemma~\ref{lem:constrained-simplex}]
$\forall C\in\Delta_{\X}^{\beta}$, $C_x\ge \delta$, $\forall x$, where $\delta=e^{1-\beta}$. Let $C^*$ be the optimal point in $\Delta_\X$, and let $C = (1-\delta) C^* + \delta \tilde{C} $ for $\tilde{C}$ to be the uniform distribution over $\X$. Then we have $C\in \Delta_{\X}^{\beta}$ and 
\begin{align*}
r_t(C^*) - r_t(C) & =  r_t(C^*) - r_t((1-\delta) C^* + \delta \tilde{C} ) \\
& \le   r_t(C^*) - (1-\delta) r_t(C^*) - \delta r_t(\tilde{C} ) & \mbox{concavity}\\
& = \delta (r_t(C^*) - r_t(\tilde{C})) \le 2 \delta \log(mn) ,
\end{align*}
where the final inequality is due to the bound on the function value $r_t(C)$. 
\end{proof}
}

\begin{lemma} 
\label{sec:ConstrainedSimplexLemma}
The following inequality holds for $\beta=2\log(mn)$ assuming $T=\tilde{O}(m+ n)$: 
\begin{align*}
\max_{C\in\Delta_{\X}}\sum_{t=1}^T r_t(C) -\max_{C\in \Delta_{\X}^\beta} \sum_{t=1}^T r_t(C) \le \tilde{O}\left(\frac{\alpha}{\min\{m, n\}}\right). 
\end{align*}
\end{lemma}
\begin{proof} [Proof of Lemma \ref{sec:ConstrainedSimplexLemma}]
Recall $\beta$ fixes an upper-bound on the gradient of the entropy of $C$. This is equivalent to fixing a lower-bound $\delta$ on the entries of $C$, where $\delta\defeq e^{1-\beta}=e(mn)^{-2}$. We define $C^*$ as follows:
\begin{align*}
C^*\defeq \argmax_{C\in\Delta_{\X}} \ \sum_{t=1}^T r_t(C).
\end{align*}
Denote the set of indices where $C^*$ is less than $\delta$ as follows: $S_\delta\defeq \{x\in\X\mid C^*_x<\delta\}$. Enumerate $S_\delta$ as $\{(i_k,j_k)\}_{k=1}^{|S_\delta|}$. 

Note that $\forall (i_k, j_k) \in S_{\delta}$, $\exists (i_k',j_k')\in\X$ such that $C^*_{i_k',j_k'}-\delta\ge \delta-C^*_{i_k,j_k}$. Otherwise, we have by choice of $\delta$,
\begin{align*}
\|C^*\|_1<mn(2\delta-C^*_{i_k,j_k})\le 2\delta mn<1. 
\end{align*}
We construct the matrix $C$ initialized as $C = C^*$. Next, for each $k \in [1, |S_{\delta}|]$ we iteratively change two entries in $C$ as follows for all $k$: 
\begin{enumerate}
    \item  $C_{i_k,j_k}=\delta $,
    \item  $C_{i_k', j_k'} = C^*_{i_k',j_k'}-(\delta-C^*_{i_k,j_k}) $.
\end{enumerate}

For each such operation,
\begin{align*}
H(C)-H(C^*)&=\left(C_{i_k,j_k}\log\left(\frac{1}{C_{i_k,j_k}}\right)+C_{i_k',j_k'}\log\left(\frac{1}{C_{i_k',j_k'}}\right)\right)\\
& \ \ \ \ \ -\left(C^*_{i_k,j_k}\log\left(\frac{1}{C^*_{i_k,j_k}}\right)+C^*_{i_k',j_k'}\log\left(\frac{1}{C^*_{i_k',j_k'}}\right)\right).
\end{align*}
Note that $C_{i_k,j_k}+C_{i_k',j_k'}=C^*_{i_k,j_k}+C^*_{i_k',j_k'}\ge 2\delta$, and $C^*_{i_k,j_k}<\delta$. Thus, $H(C)-H(C^*)\ge 0$ holds for each operation. $C$ constructed by this enumeration satisfies $H(C)\ge H(C^*)$. On the other hand, $G$ is Lipschitz. In particular, let $\bar{M}^1=\frac{1}{T} \sum_{t=1}^T M_t^1$, $\bar{M}^2=\frac{1}{T} \sum_{t=1}^T M_t^2$, then
\begin{align*}
\sum_{t=1}^T G(C,M_t^1,M_t^2)-\sum_{t=1}^TG(C^*,M_t^1,M_t^2)&= \alpha T \sum_{i,j}(C_{ij}-C^*_{ij})(\bar{M}^1_{ij}-\bar{M}^2_{ij}) &\text{linearity of $G$}\\
&\le 2 \alpha T\|C-C^*\|_1 & \bar{M}^1,\bar{M}^2\in[-1,1]^{\X}\\
&\le 4 mn \delta \alpha T &\|C-C^*\|_{\infty}\le\delta\\
&=\frac{4e \alpha T}{mn}. &\delta=\frac{e}{(mn)^2} 
\end{align*}
Under the assumption that $T\ge\tilde{O}(m+n)$, we conclude that
$$\max_{C\in\Delta_{\X}}\sum_{t=1}^T r_t(C)-\max_{C\in\Delta_{\X}^{\beta}}\sum_{t=1}^T r_t(C)\le \tilde{O}\left(\frac{\alpha}{\min\{m,n\}}\right).$$
\end{proof}

By taking $\beta=O(\log(mn))$ and assuming $T=\tilde{O}(m+n)$, we can conclude that

\begin{align*}
\regret_T^{\A_C}\defeq \max_{C\in\Delta_{\X}}\sum_{t=1}^T r_t(C)-\sum_{t=1}^T r_t(C_t)\le O(\alpha\sqrt{\log(mn)T})= \tilde{O}(\alpha\sqrt{T}).
\end{align*}
\end{proof}

\subsubsection{Proof of Theorem~\ref{thm:M-subroutine-regret}}
\label{app:M-subroutine-proof}
\begin{reptheorem}{thm:M-subroutine-regret}
Denote at every time step $t$, consider the concave reward function $\gamma_t(M)\defeq G(C_t,M)-\theta f_t(M)$. There exists sub-routine gradient-based update $\A_M$ (see Alg.~\ref{appalg:omd-m} for an example) such that, under either setting \ref{assumptionH1} or \ref{assumptionH2}, the following regret guarantee w.r.t. $\gamma_t$ holds:
\begin{align*}
\regret_T^{\A_M}\defeq\max_{M\in\mMmax\times\mMmax}\sum_{t=1}^T \gamma_t(M)-\sum_{t=1}^T \gamma_t (M_t)\le O(K \theta \sqrt{(m+n)T}).
\end{align*}
\end{reptheorem}

\begin{proof}[Proof of Theorem~\ref{thm:M-subroutine-regret}]

We will begin by outlining the update algorithm $\A_M$, then introduce the key definitions used in $\A_M$, and proceed to prove Theorem \ref{thm:M-subroutine-regret}. We note that this algorithm is modified from the matrix multiplicative weights for online matrix prediction (Algorithm 2) in \cite{hazan2012near}.

\begin{algorithm}[h!]
\caption{$\A_M$}
\label{appalg:omd-m}
\begin{algorithmic}[1]
\STATE Input: $C_t,M_t=(M_t^1,M_t^2)\in\mMmax$, $(i_t,j_t,o_t)$.
\IF{input is empty}
\STATE Output: $\phi^{-1}\left(\frac{K}{2}I\right)$. 
\ENDIF
\STATE Compute $X_t=\phi(M_{t}^1,M_{t}^2)$ with $\phi(\emptyset)=\frac{K}{2}I$. 
\STATE Create matrix $L_t(\gamma_t)$ according to Definition~\ref{def:descentmatrix}.
\STATE Update: with step size $\eta=\frac{1}{2(1+8\theta)}\sqrt{\frac{(m+n)\log(2p)}{T}}$, project w.r.t. matrix relative entropy:
\begin{align*}
X_{t+1}=\underset{X\in\K_X}{\argmin} \ \Delta(X,\exp(\log(X_t)+\eta L_t(\gamma_t))).
\end{align*}
\STATE Output: $M_{t+1}=(M_{t+1}^1,M_{t+1}^2)=\phi^{-1}(X_{t+1})$. 
\end{algorithmic}
\end{algorithm}

$\K_X$ is given by
\begin{align*}
\K_X\defeq \{X\in \sym(2p): \ & X\succeq 0, \ \tr(X)\le 2K(m+n), \ X_{ii}\le K,\ X[:p,:p]-X[p:,p:]\in[-1,1]^{p\times p}\}.
\end{align*}

\ignore{
For convenience, with slight abuse of notation, let 
$$M_t\defeq \begin{bmatrix}M_t^1 & \mathbf{0}\\\mathbf{0} & M_t^2\end{bmatrix}.$$
}

The operator $\phi$ is given as the following:

\begin{lemma} 
\label{app:decomposabilitylemma}
$\forall M^1,M^2\in\mMmax$, $M=\begin{bmatrix}
M^1 & \mathbf{0}\\
\mathbf{0} & M^2
\end{bmatrix}$ is $(K,2K(m+n))$-decomposable. \ignore{In particular, this means that every $M_t$ is $(K,2K(m+n))$-decomposable. }

\end{lemma}

Denote $p\defeq 2(m+n)$. $(\beta,\tau)$-decomposability allows for two matrices in $\mMmax$ to be embedded in $\sym(2p)$:

$\phi(\cdot,\cdot):\mMmax\times\mMmax\rightarrow\sym(2p)$ is the embedding operator given by
$$\phi(M)\defeq\phi(M^1,M^2)\defeq \begin{bmatrix}P & \mathbf{0}\\\mathbf{0} & N\end{bmatrix},$$
where $P,N$ are the PSD matrices given by the $(\beta,\tau)$-decomposition. 

The descent matrix $L_t(\gamma_t)$, which we shorthand denote as $L_t$, is constructed as the following:
\begin{definition} [Descent matrix]
\label{def:descentmatrix}
At time $t$, we define the matrix $L_t\in \sym(2p)$ as 
$L_t\defeq L_t^G+L_t^F$, where
$L_t^G$ is symmetric and
$$\begin{cases}
L_t^G[1:m,2m+1:2m+n]=C_t\\
L_t^G[p+m+1:p+2m,p+2m+n+1:2p]=C_t\\
L_t^G[m+1:2m,2m+n+1:p]=-C_t\\
L_t^G[p+1:p+m,p+2m+1:p+2m+n]=-C_t\\
L_t^G[i,j]=0 \hspace{0.1cm} \text{if otherwise and $j\ge i$}
\end{cases}$$

$L_t^F$ is symmetric and
$$\begin{cases}
L_t^F[i_t,2m+j_t]=-2\theta ((M_t^1)_{i_t,j_t}-o_t),L_t^F[m+i_t,2m+n+j_t]=-2\theta ((M_t^2)_{i_t,j_t}-o_t)\\
L_t^F[p+i_t,p+2m+j_t]=2\theta((M_t^1)_{i_t,j_t}-o_t), L_t^F[p+m+i_t,p+2m+n+j_t]=2\theta((M_t^2)_{i_t,j_t}-o_t)\\
L_t^F[i,j]=0 \hspace{0.1cm} \text{if otherwise and $j\ge i$}
\end{cases}$$

Note that by construction $(L_t^G)^2$, $(L_t^F)^2$ are diagonal matrices, and $\tr(L_t^2)\le O(\theta^2)$. 
\end{definition}

The rest of the proof follows from Section 3.2 in \citep{hazan2012near}.

\end{proof}

\subsection{Proof of offline implications}
\subsubsection{Proof of Lemma~\ref{thm:online-to-offline}}
\label{sec:ProofMainTheorem}
\begin{replemma}
{thm:online-to-offline} 
After $T$ iterations, and assume that for some $\delta>0$,
$$ \frac{1}{\theta}\left(2D+ \frac{\regret_T^{\mA_M}}{T}\right) \le \frac{\delta^{2/3}}{2},$$ 
with $D=1$ in setting \ref{assumptionH1}, and $D=\sqrt{m+n}$ in setting \ref{assumptionH2}. 
The following properties hold on the obtained $\bar{C}\defeq \frac{1}{T}\sum_{t=1}^T C_t$ returned by Algorithm \ref{alg:odd}: with probability $\ge 1-\exp\left(-\frac{\delta^{4/3}T}{512}\right)$,
\begin{align*}
H(\bar{C}) -   \max_{M\in \V_{T,\delta}^2} \alpha \cdot G(\bar{C},M)\ge \max_{C\in\C}\min_{M\in\V_{\delta^{2/3}}^2} \left\{ H(C)-  \alpha \cdot G(C,M)\right\}-2\alpha\theta\delta- \frac{B_T} {T}.
\end{align*}
\end{replemma}

\begin{proof}[Proof of Lemma~\ref{thm:online-to-offline}]
For notation convenience, denote
\begin{align*}
C^{\star}, M^1_{\star}, M^2_{\star}=\underset{C\in\C}{\argmax} \ \underset{(M^1,M^2)\in\V_{\delta^{2/3}}^2}{\argmin} \ H(C)-\alpha G(C,M^1,M^2).
\end{align*}
Consider the subroutine $\A_M$. Under the assumption, $\A_M$ is a low-regret OCO algorithm for $\gamma_t$'s. In particular, under the realizable assumption, since there exist $M=(M^1,M^2)$ such that $f_t(M)=\gamma_t(M)=0$, $\forall t$, we have that for the sequence of $M_t$'s output by the algorithm,
\begin{align*}
\frac{1}{T}\sum_{t=1}^T f_t(M_t)& =  \frac{1}{T \theta} \sum_{t=1}^T \left( G(C_t,M_t)-\gamma_t(M_t)\right)\le\frac{1}{\theta}\left(2D+\frac{\regret_T^{\A_M}}{T}\right)\le \frac{\delta^{2/3}}{2}. 
\end{align*}
 Note that $\forall M\in\mMmax$, $M$ can be seen as a function mapping from $\X$ to $[-1,1]$. Denote $f(M)\defeq\E_{(x,o)\sim\D}[(M_x^1-o)^2+(M_x^2-o)^2]$. Define $Z_t=f(M_t)-f_t(M_t)$, $X_t=\sum_{i=1}^t Z_i$, then we have with $\F_t$ denoting the filtration generated by the algorithm's randomness up to iteration $t$, and since $M_t\in\F_{t-1}$,
\begin{align*}
\E[Z_t\mid\F_{t-1}]=0, \ \ \E[X_t\mid\F_{t-1}]=X_{t-1}, 
\end{align*}
and $|X_t-X_{t-1}|=|Z_t|\le 8$. By Azuma's inequality, we have $\forall \eps>0$,
\begin{align*}
\mathbb{P}\left(\frac{1}{T}\sum_{t=1}^Tf(M_t)-f_t(M_t)>\sqrt{\frac{128\log\left(\frac{1}{\eps}\right)}{T}}\right)=\mathbb{P}\left(\frac{1}{T}\sum_{t=1}^T X_t>\sqrt{\frac{128\log\left(\frac{1}{\eps}\right)}{T}}\right)\le\eps.
\end{align*}
We can conclude that with probability at least $1-\exp\left(-\frac{\delta^{4/3}T}{512}\right)$, 
\begin{align*}
f(\bar{M})\le \frac{1}{T}\sum_{t=1}^T f(M_t)\le \frac{1}{T}\sum_{t=1}^T f_t(M_t)+\frac{\delta^{2/3}}{2}\le\frac{\delta^{2/3}}{2}+\frac{\delta^{2/3}}{2}=\delta^{\frac{2}{3}},
\end{align*}
in which case $\bar{M}^1,\bar{M}^2\in\V_{\delta^{2/3}}$.
\ignore{
First, we show that $\E[\bar{M}^1], \E[\bar{M}^2]\in\V_{\delta}$. Note that, $\forall t$,
\begin{align*}
\E[f_t(\E[\bar{M}^1],\E[\bar{M}^2])]\le \E[f_t(\bar{M}^1,\bar{M}^2)]&=\frac{1}{\theta}\E\left[G(C_t,\bar{M}^1,\bar{M}^2)-\gamma_t(\bar{M}^1,\bar{M}^2)\right] &\text{definition of $\gamma_t$}\\
&\le \frac{1}{\theta} \E\left[G(C_t,\bar{M}^1,\bar{M}^2)-\frac{1}{T}\sum_{t=1}^T \gamma_t(M_t^1,M_t^2)\right] &\text{concavity of $\gamma_t$}\\
&\le\frac{1}{\theta}\left(\E\left[G(C_t,\bar{M}^1,\bar{M^2})\right]+\frac{\regret_T^{\A_M}}{T} \right)&\text{$\A_M$ regret guarantee}\\
&\le \frac{1}{\theta}\left(2D+\frac{\regret_T^{\A_M}}{T}\right) &\text{$D$ is the bound of $\|\cdot\|_1$ on $\C$}\\
&\le \delta &\text{assumption}
\end{align*}
Take expectation on both side over $(i,j,o)\sim\D$, we conclude that $\E[\bar{M}^1],\E[\bar{M}^2]\in\V_{\delta}$. 
}
We have thus with probability at least $1-\exp\left(-\frac{\delta^{4/3}T}{512}\right)$, 
\begin{align*}
&H(C^\star)-\alpha G(C^\star,M^1_\star,M^2_\star)\\
&\le H(C^\star)-\alpha G(C^\star, \bar{M}^1,\bar{M}^2) &\text{($M^1_{\star}, M^2_{\star}$) are optimal w.r.t. $C^\star$ in $\V_{\delta}$}\\
&=\frac{1}{T}\sum_{t=1}^T r_t(C^\star)&\text{linearity of $G$, definition of $r_t$}\\
&\le \frac{1}{T} \sum_{t=1}^T r_t(C_t)+\frac{\regret_T^{\A_C}}{T}&\text{$\A_C$ regret guarantee}\\
&\le H(\bar{C})-\frac{\alpha}{T}\sum_{t=1}^TG(C_t,M_t^1,M_t^2)+\frac{\regret_T^{\A_C}}{T} &\text{concavity of $H(\cdot)$ on $\C$}\\
&\le H(\bar{C})-\frac{\alpha}{T}\sum_{t=1}^T \gamma_t(M_t^1,M_t^2)+\frac{\regret_T^{\A_C}}{T} &\text{$f_t\ge 0$}\\
&\le  H(\bar{C})-\frac{\alpha}{T}\sum_{t=1}^T \gamma_t(\hat{M}^1,\hat{M}^2)+\frac{B_T}{T} &\text{$\forall (\hat{M}^1,\hat{M}^2)\in{\mMmax}^2$ by $\A_M$ regret guarantee}\\
& \leq H(\bar{C})-\max_{(M^1,M^2)\in\V_{T,\delta}^2} \alpha\cdot G(\bar{C},M^1,M^2)+2\alpha\theta\delta+\frac{B_T}{T} &\text{definition of $\V_{T,\delta}$}
\end{align*}
\end{proof}

\subsubsection{Proof of Theorem~\ref{cor:online-to-offline}}
\label{app:online-to-offline-cor-proof}
\begin{reptheorem}{cor:online-to-offline}
Suppose the underlying sampling distribution is $\mu$. Let $C_{\mu} \in\C$ be its corresponding confidence matrix. In particular, for setting \ref{assumptionH1}, $C_{\mu}=\mu$, i.e.  $(C_{\mu})_{ij}=\mathbb{P}_{\mu}((i,j)\text{ is sampled})$; for setting \ref{assumptionH2}, $C_{\mu}$ satisfies that $C_{\mu}/\|C_{\mu}\|_1=\mu$. Then, for any $\delta>0$, Algorithm~\ref{alg:odd} run with $\alpha=\delta^{-1/6}$ returns a $\bar{C}$ that guarantees the following bounds: with probability $\ge 1-\exp\left(-\frac{\delta^2T}{128}\right)-\exp\left(-\frac{\delta^{4/3}T}{512}\right)$,
\begin{enumerate}
\item For setting \ref{assumptionH1}, take $\theta=4\delta^{-2/3}$, after $T=\tilde{O}(\delta^{-2}K^2(m+n))$ iterations,
\begin{enumerate}
\item $H(\bar{C})\ge H(C_{\mu})-O(\delta^{1/6})$.
\item $\underset{M^1,M^2\in\V_{\frac{\delta}{2}}}{\max}\ G(\bar{C},M^1,M^2)\le O(\delta^{1/6}\log(mn))$.
\end{enumerate}
\item For setting \ref{assumptionH2}, take $\theta=4\delta^{-2/3}\sqrt{m+n}$, after $T=\tilde{O}(\delta^{-2}K^2(m+n))$ iterations, 
\begin{enumerate}
\item $\|\bar{C}\|_1\ge \|C_{\mu}\|_1-O(\delta^{1/6}\sqrt{m+n})$.
\item $\underset{M^1,M^2\in\V_{\frac{\delta}{2}}}{\max} G(\bar{C},M^1,M^2)\le O(\delta^{1/6}\sqrt{m+n})$.
\end{enumerate}
\end{enumerate}
\end{reptheorem}
\begin{proof}[Proof of Theorem~\ref{cor:online-to-offline}]
First, note that $\forall i,j,o$, we have by assumption $(M_{ij}-o)^2\in[0,4]$, $\forall M\in\mMmax$. Therefore, by subgaussian concentration, $\forall c\ge 1$,
\begin{align*}
\mathbb{P}\left(\frac{1}{T}\sum_{t=1}^T (M_{i_t,j_t}-o_t)^2 - \E_{i,j}[(M_{ij}-o)^2]\ge \frac{\delta}{2}\right)\le\exp\left(-\frac{\delta^2 T}{128}\right).
\end{align*}
Therefore, with probability at least $1-\exp\left(-\frac{\delta^2T}{128}\right)$, we have
\begin{align*}
\V_{\frac{\delta}{2}}\subseteq \V_{T,\delta}. 
\end{align*}
Therefore, it suffices to show the inequality in (b) for the maximum over $\V_{T,\delta}$. Let
\begin{align*}
C^{\star}, M^1_{\star}, M^2_{\star}&=\underset{C\in\C}{\argmax} \ \underset{(M^1,M^2)\in\V_{\delta^{2/3}}^2}{\argmin} \ H(C)-\alpha G(C,M^1,M^2)\\
&=\underset{C\in\C}{\argmax} \ \underset{(M^1,M^2)\in\V_{\delta^{2/3}}^2}{\argmin} \ H(C)-\delta^{-1/6} G(C,M^1,M^2).
\end{align*}
Choose $T$ such that $\frac{B_T}{T}\le \alpha\theta\delta$. Note that in both settings, the choice of $\theta$ and $T$ satisfies the assumption in Theorem~\ref{thm:online-to-offline}. 
\paragraph{Simplex and entropy.}
We can bound $G(C_{\mu},M_{\star}^1,M_{\star}^2)$ by
\begin{align*}
G(C_{\mu},M^1_{\star},M^2_{\star})&=\E_{x\sim\mu}[(M_{\star}^1)_x-(M_{\star}^2)_x]\\
&=\sqrt{(\E_{(x,o)\sim\D}[(M_\star^1)_{x}-o_{x}]-\E_{(x,o)\sim\D}[(M_\star^2)_{x}-o_{x}])^2}\\
&\le \sqrt{2(\E_{(x,o)\sim\D}[(M_\star^1)_{x}-o_{x}]^2+\E_{(x,o)\sim\D}[(M_\star^2)_{x}-o_{x}]^2)} &(a-b)^2\le 2(a^2+b^2)\\
&\le \sqrt{2(\E_{(x,o)\sim\D}[((M_\star^1)_{x}-o_{x})^2]+\E_{(x,o)\sim\D}[((M_\star^2)_{x}-o_{x})^2])} &\text{Jensen's}\\
&\le 2\delta^{1/3} &M_\star^1, M_\star^2\in\V_{\delta^{2/3}}
\end{align*}
Theorem \ref{thm:online-to-offline} implies that with probability $\ge 1-\exp\left(-\frac{\delta^{4/3}T}{512}\right)$,
\begin{align*}
H(\bar{C})-\underset{M^1,M^2\in\V_{T,\delta}}{\max} \ \delta^{-1/6}G(\bar{C},M^1,M^2) 
& \ge H(C_{\mu}) -\alpha G(C_{\mu},M_{\star}^1, M_{\star}^2) - 2\alpha\theta\delta -\frac{B_T}{T}\\
&\ge H(C_{\mu})-12\delta^{1/6},
\end{align*}
Note that by definition
$\underset{M^1,M^2\in\V_{T,\delta}}{\max} \ G(\bar{C},M^1,M^2)\ge 0$, and thus
\begin{align*}
H(\bar{C})\ge H(C_{\mu})-12\delta^{1/6}.
\end{align*}
Since $H(\cdot)$ is bounded by $\left[0,\log(mn)\right]$ over $\C$, then
\begin{align*}
\underset{M^1,M^2\in\V_{T,\delta}}{\max} G(\bar{C},M^1,M^2)&\le \delta^{1/6}\left(H(\bar{C})-H(C_{\mu})\right)+12\delta^{1/3}\\
&\le \delta^{1/6}\log(mn)+12\delta^{1/3}\\
&\le O(\delta^{1/6}\log(mn)).
\end{align*}

\paragraph{Cube and $\ell_1$ norm.}
We can bound $G(C_{\mu},M_{\star}^1,M_{\star}^2)$ by 
\begin{align*}
G(C_{\mu},M_{\star}^1,M_{\star}^2) \le \sqrt{m+n}\E_{x\sim\mu}\left[(M_{\star}^1)_x-(M_{\star}^2)_x\right]\le 2\delta^{1/3}\sqrt{m+n}.
\end{align*}
Theorem \ref{thm:online-to-offline} implies that 
\begin{align*}
H(\bar{C})-\underset{M^1,M^2\in\V_{T,\delta}}{\max} \ \delta^{-1/6}G(\bar{C},M^1,M^2)\ge H(C_{\mu})-12\delta^{1/6}\sqrt{m+n}.
\end{align*}
Thus, similar as before, we get
\begin{align*}
H(\bar{C})&\ge H(C_{\mu})-12\delta^{1/6}\sqrt{m+n},
\end{align*}
and
\begin{align*}
\underset{M^1,M^2\in\V_{T,\delta}}{\max} G(\bar{C},M^1,M^2)&\le \delta^{1/6}\left(H(\bar{C})-H(C_{\mu})\right)+12\delta^{1/3}\sqrt{m+n}\\
&\le O(\delta^{1/6}\sqrt{m+n}).
\end{align*}
\end{proof}

\section{Implementation Details}
With the theoretical guarantee established above, we implemented a simple version of ODD in Section~\ref{sec:experiments} and experimented with both simulated toy datasets and real-world datasets including the well-known MovieLens dataset. In particular, the simplification of ODD we implemented is given by the following algorithm:

\begin{algorithm}[!ht]
\caption{Simplified ODD}
\label{sec:heuristic}
\begin{algorithmic}[1]
\STATE \textbf{Input}: initial (uniform) distribution $C_1$, matrix $M\in\K$, parameters $\eta>0$, $\alpha>0$, matrix prediction update function \texttt{matrix-predict}($\cdot$). $R(\cdot)=-H(\cdot)$.
\FOR {$t=1,2,...,T$}
\STATE Adversary draws tuple $x_t=(i_t,j_t)$, reveals $M_{x_t}$.
\STATE Update $M_{t+1} = \texttt{matrix-predict}(C_t,M_t)$.
\STATE Consider reward 
$\tilde{r}_t(C)\defeq \alpha H(C)- \langle C,M_t\rangle.$ 
\STATE Update {
$\nabla R(\hat{C}_{t+1})\leftarrow\nabla R(C_t)+\eta\nabla \tilde{r}_t(C_t).$ 
}
\STATE Project 
$C_{t+1}=\underset{C\in\Delta_{\X}}{\argmin} \ \frnorm{C-\hat{C}_{t+1}}$.
\ENDFOR
\STATE \textbf{return}: $\bar{C}\defeq\frac{1}{T}\sum_{t=1}^T C_t$.
\end{algorithmic}
\end{algorithm}

This algorithm is an instantiation of our online algorithm ODD (Algorithm~\ref{alg:odd}), using mirror descent for updating the confidence matrix, and an arbitrary matrix completion method called \texttt{matrix-predict}.  

\section{Definitions}
\label{app:definitions}
\paragraph{Matrix logarithm.}
Given the matrix $X\in \sym(n)$, $X\succeq 0$, $X$ admits a diagonalization $X=V\Sigma V^T$, where $\Sigma$ can be written as follows:
$$\Sigma\defeq \begin{bmatrix}
\Sigma_{11} & 0 & ... & 0\\
0 & \Sigma_{22} & ... & 0\\
... & ... & ... & ...\\
0 & 0 & ... & \Sigma_{nn}
\end{bmatrix}.$$
The logarithm of $X$ is given by:
$$\log(X)\defeq V \begin{bmatrix}
\log(\Sigma_{11}) & 0 & ... & 0\\
0 & \log(\Sigma_{22}) & ... & 0\\
... & ... & ... & ...\\
0 & 0 & ... & \log(\Sigma_{nn})
\end{bmatrix} V^T.$$
Matrix logarithm satisfies the following properties for $X,Y\in\sym(n), X,Y\succ 0$:
\begin{enumerate}
\item $\tr(\log(XY))=\tr(\log(X))+\tr(\log(Y))$.
\item If $XY=YX$, then $\log(XY)=\log(X)+\log(Y)$. \label{app:normalmatrixlog}
\item $\log(cI)=(\log c) I$, $\forall c>0$. \label{app:scalematrixlog}
\end{enumerate}

\paragraph{Matrix relative entropy.}
Given matrices $X,Y\in\sym(n), X,Y\succeq 0$, their relative entropy is given by 
$$\Delta(X, Y) \defeq  \text{Tr}(X \log(X) - X \log(Y) - X + Y).$$

\paragraph{($\beta,\tau$)-decomposability.}
$M\in\mathbb{R}^{m\times n}$ is $(\beta,\tau)$-decomposable if $\exists P,N\in \sym(m+n)$, $P,N\succeq 0$:
\begin{align*}
P-N=\begin{bmatrix}
\mathbf{0} & M\\
M^T & \mathbf{0}
\end{bmatrix}, \ \ \ \tr(P)+\tr(N)\le \tau, \ \ \ \max_{i} \ P_{ii}, \max_{i} \ N_{ii}\le \beta.
\end{align*}

\newpage

\end{document}